%% file: main.tex
\documentclass{article}

\usepackage{arxiv}

\usepackage[utf8]{inputenc} 
\usepackage[T1]{fontenc}    
\usepackage{hyperref}       
\usepackage{url}            
\usepackage{booktabs}       
\usepackage{amsfonts}       
\usepackage{nicefrac}       
\usepackage{microtype}      
\usepackage{lipsum}		
\usepackage{graphicx}
\usepackage{subfig}
\usepackage{natbib}
\usepackage{doi}
\usepackage{amssymb}
\usepackage{amsmath}
\usepackage{latexsym}
\usepackage{amsthm}
\usepackage{colortbl}
\usepackage{makecell}
\usepackage{array}
\usepackage{algorithm}%
\usepackage{algorithmic}%
\definecolor{newcolor}{rgb}{.8,.349,.1}
\definecolor{myblue}{rgb}{.294,.616,.863}
\usepackage{thmtools, thm-restate}
\usepackage[normalem]{ulem}
\useunder{\uline}{\ul}{}
\usepackage{adjustbox}
\usepackage{enumitem}

\usepackage{algorithmic}
\newcommand{\comment}[1]{\textit{\color{myblue}\# #1}}

\usepackage{arydshln}

\title{Reliable uncertainty quantification for 2D/3D anatomical landmark localization using multi-output conformal prediction}

\author{ \href{https://orcid.org/0000-0003-3608-0308}{\includegraphics[scale=0.06]{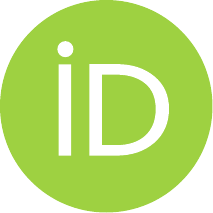}\hspace{1mm}Jef Jonkers}\\
	IDLab \\
        Department of Electronics and Information Systems \\
        Ghent University, Belgium \\
	\texttt{jef.jonkers@ugent.be} \\
    \And
    \href{https://orcid.org/0000-0002-7617-3210}{\includegraphics[scale=0.06]{orcid.pdf}\hspace{1mm}Frank Coopman}\\
	Biometrics Research Group \\ 
        Department of Morphology, Imaging, Orthopedics, \\
        Rehabilitation and Nutrition \\
        Ghent University, Belgium \\
        \texttt{frank.coopman@ugent.be} \\
	\And
    \href{https://orcid.org/0000-0003-0462-3638}{\includegraphics[scale=0.06]{orcid.pdf}\hspace{1mm}Luc Duchateau}\\
	Biometrics Research Group \\ 
        Department of Morphology, Imaging, Orthopedics, \\
        Rehabilitation and Nutrition \\
        Ghent University, Belgium \\
        \texttt{luc.duchateau@ugent.be} \\
	\And
    \href{https://orcid.org/0000-0001-9530-3466}{\includegraphics[scale=0.06]{orcid.pdf}\hspace{1mm}Glenn Van Wallendael}\\
	IDLab \\
        Department of Electronics and Information Systems \\
        Ghent University - imec, Belgium \\
	\texttt{glenn.vanwallendael@ugent.be} \\
	\And
	\href{https://orcid.org/0000-0002-7865-6793}{\includegraphics[scale=0.06]{orcid.pdf}\hspace{1mm}Sofie Van Hoecke}\\
	IDLab \\
        Department of Electronics and Information Systems \\
        Ghent University - imec, Belgium \\
	\texttt{sofie.vanhoecke@ugent.be} \\
}

\renewcommand{\shorttitle}{Reliable uncertainty quantification for landmark localization using conformal prediction}

\hypersetup{
pdftitle={Reliable uncertainty quantification for 2D/3D anatomical landmark localization using multi-output conformal prediction},
pdfsubject={cs.LG, stat.ML, cs.CV, cs.AI},
pdfauthor={Jef Jonkers, Frank Coopman, Luc Duchateau, Glenn Van Wallendael, Sofie Van Hoecke},
pdfkeywords={Landmark Localization, Uncertainty Quantification, Conformal Prediction, Multi-output Conformal Prediction, Prediction Region},
}

\begin{document}
\maketitle

\begin{abstract}
	Automatic anatomical landmark localization in medical imaging requires not just accurate predictions but reliable uncertainty quantification for effective clinical decision support. Current uncertainty quantification approaches often fall short, particularly when combined with normality assumptions, systematically underestimating total predictive uncertainty. This paper introduces conformal prediction as a framework for reliable uncertainty quantification in anatomical landmark localization, addressing a critical gap in automatic landmark localization. We present two novel approaches guaranteeing finite-sample validity for multi-output prediction: Multi-output Regression-as-Classification Conformal Prediction (M-R2CCP) and its variant Multi-output Regression to Classification Conformal Prediction set to Region (M-R2C2R). Unlike conventional methods that produce axis-aligned hyperrectangular or ellipsoidal regions, our approaches generate flexible, non-convex prediction regions that better capture the underlying uncertainty structure of landmark predictions. Through extensive empirical evaluation across multiple 2D and 3D datasets, we demonstrate that our methods consistently outperform existing multi-output conformal prediction approaches in both validity and efficiency. This work represents a significant advancement in reliable uncertainty estimation for anatomical landmark localization, providing clinicians with trustworthy confidence measures for their diagnoses. While developed for medical imaging, these methods show promise for broader applications in multi-output regression problems.
\end{abstract}


\section{Introduction}
Automatic anatomical landmark localization is crucial in advancing diagnostic capabilities based on medical images. Despite the increasing volume of publications within the biomedical domain addressing this area \citep{payer_integrating_2019, thaler_modeling_2021, ryu_automated_2022, pei_learning-based_2023, de_queiroz_tavares_borges_mesquita_artificial_2023}, the current emphasis on maximizing predictive accuracy neglects uncertainty quantification. Uncertainty quantification is vital for unlocking the full potential of clinical artificial intelligence by providing a more nuanced understanding of model predictions, enhancing diagnostic accuracy and decision support, and subsequently, improving human health \citep{banerji_clinical_2023}.

The central emphasis on predictive accuracy prompts the need to focus on the reliability of machine learning models, especially in clinical settings. Deployed models must not only deliver accurate predictions but also provide reliable uncertainty measures, increasing trust in their outputs and enabling decision support frameworks \citep{holzinger_interactive_2016}.

Existing approaches to uncertainty quantification in anatomical landmark localization often fall short of providing reliable estimates \citep{kumar_uglli_2019, kumar_luvli_2020, drevicky_evaluating_2020, lee_automated_2020, ma_volumetric_2020, payer_uncertainty_2020,thaler_modeling_2021, kwon_multistage_2021, schobs_uncertainty_2023,schobs_bayesian_2023}, i.e., failing to accurately quantify the degree of confidence in predictions or measurements, which is a fundamental characteristic of any uncertainty quantification method. Approaches often also fail to address the total predictive uncertainty \citep{kumar_uglli_2019, kumar_luvli_2020, ma_volumetric_2020, lee_automated_2020, kwon_multistage_2021, mccouat_contour-hugging_2022, schobs_uncertainty_2023}, i.e., encompassing all sources of variability and error. Consequently, there is a need to explore alternative methodologies that offer robust uncertainty estimates in medical image analysis.

We propose introducing conformal prediction to anatomical landmark localization tasks, a methodology capable of generating prediction regions with finite sample coverage guarantees. Besides evaluating several existing multi-output conformal prediction approaches for the landmark localization task, this work also introduces two novel methods, i.e., Multi-output Regression to Classification conformal prediction set to Region (M-R2C2R) and Multi-output Regression-as-Classification Conformal Prediction (M-R2CCP), an extension of the R2CCP \citep{guha_conformal_2023} approach to multi-output regression. These approaches successfully overcome key challenges in landmark localization and in multi-output regression problems in general, including handling multi-dimensional outputs effectively and managing varying response domains across different examples.

The remainder of the paper is structured as follows. Section \ref{sec:related-work} provides a comprehensive review of existing uncertainty quantification approaches for landmark localization, highlighting their limitations in accounting for different sources of uncertainty and the lack of reliability guarantees in current methods. Section \ref{sec:uq-general} describes different uncertainty quantification approaches for landmark localization and introduces our novel multi-output conformal prediction approaches.  Section \ref{sec:datasets} discusses the experimental setup and datasets, followed by Section \ref{sec:experiments}, which discusses and presents the results. Finally, Section \ref{sec:conclusion} concludes the work and outlines directions for future research.

\section{Related work}
\label{sec:related-work}
Below, we briefly discuss different landmark localization approaches and the related work to uncertainty quantification for this problem.

\subsection{Landmark localization}
Landmark localization in medical imaging has evolved significantly with the emergence of deep learning approaches. Mainly, approaches have been based on coordinate and heatmap regression~\citep{jonkers2025landmarkertoolkitanatomicallandmark}. While coordinate regression directly predicts landmark coordinates, heatmap regression has gained prominence since its introduction by \citet{tompson_joint_2014}, who demonstrated superior performance by predicting likelihood maps for each landmark. However, heatmap-based models generally perform worse regarding inference time and memory consumption than coordinate regression approaches~\citep{fard_acr_2022}. 

Recent advances in heatmap regression have led to several methodological variations. Direct heatmap regression methods can be categorized into static and adaptive approaches~\citep{jonkers2025landmarkertoolkitanatomicallandmark}. Static approaches use fixed hyperparameters for heatmap distribution during training, while adaptive methods dynamically adjust these parameters. Some works have explored treating heatmap parameters as learnable model parameters \citep{payer_integrating_2019, payer_uncertainty_2020, thaler_modeling_2021}, while others have implemented scheduling methods that modify parameters based on evaluation metrics \citep{teixeira_adaloss_2019}.

The spatial configuration network \citep{payer_integrating_2019} has emerged as a particularly effective architecture for medical imaging applications where anatomical structures maintain consistent spatial relationships, such as cephalograms or pelvis radiographs. This approach integrates spatial configuration information into the heatmap regression framework, improving localization accuracy for positionally consistent medical images.

Another approach that allows for a better representation of predictive uncertainty is the contour-hugging method \citep{mccouat_contour-hugging_2022}, which transforms landmark localization into a classification task using one-hot encoded heatmaps.

Fully convolutional neural networks with differentiable decoding operations, such as the soft-argmax operation \citep{luvizon_2d3d_2018, dong_supervision-by-registration_2018, bulat_subpixel_2021}, generate heatmaps as an intermediate layer and optimize using coordinate-based loss. Some approaches further account for annotation ambiguity through uncertainty-aware loss functions \citep{kumar_uglli_2019, kumar_luvli_2020, zhou_star_2023}, enhancing the robustness of landmark predictions.

Two-stage approaches have proven particularly effective in the medical domain \citep{jiang_cephalformer_2022, song_automatic_2020, zhong_attention-guided_2019}. These methods first localize landmarks on low-resolution images before refining predictions on high-resolution patches centered around initial estimates. These approaches have been successfully applied across various medical imaging tasks, including cephalometric landmark localization \citep{song_automatic_2020}.

\subsection{Uncertainty quantification}
\label{sec:related-work-uq}
As stated before, the literature on uncertainty quantification for landmark localization is limited. The available work (see Table \ref{tab:overview-works}) lacks reliability guarantees or does not properly account for all sources of uncertainty, both aleatoric (data-related) and epistemic (model-related).

\subsubsection{Aleatoric and epistemic uncertainty}
\paragraph{Aleatoric uncertainty} Data-related or aleatoric uncertainty refers to the irreducible part of the uncertainty due to the stochastic dependency between the input instances and outcomes or the lack of explanatory depth, i.e., there is a missing explanatory variable that could make the outcomes entirely separable. Thus, there is no aleatoric uncertainty for a prediction problem with a deterministic dependency between the input object and the outcome. However, in the landmark localization task, aleatoric uncertainty is present due to annotation errors and the ambiguity of specific landmarks. As a result, medical diagnoses that rely on detecting landmarks show high inter-rater and intra-rater variability \citep{kamoen_clinical_2001}, resulting in possible misdiagnosis. It is important to stress that machine learning alone cannot entirely remove this variability. However, it can reduce human errors, which have nothing to do with the ambiguity of the landmarks and can still quantify the uncertainty in a diagnosis, which clinicians can consider. Quantifying this aleatoric uncertainty can give insights into the reasons for misdiagnosis and indicate how to revise current medical diagnosis procedures. 

Aleatoric uncertainty can, for example \citep{hullermeier_aleatoric_2021}, be modeled by performing maximum likelihood inference, for instance, by training a neural network that, next to trying to estimate the target value based on the input value, also tries to predict the heteroscedastic aleatoric uncertainty \citep{kendall_what_2017}. Typically, the Gaussian negative log-likelihood is used as a loss function that also tries to capture the data-related uncertainty; however, other loss functions exist \citep{kendall_what_2017}. Note that using a specific loss function introduces an inductive bias to the aleatoric uncertainty.

\paragraph{Epistemic uncertainty} Model and approximation uncertainty are regarded as epistemic uncertainty, meaning uncertainty due to a lack of knowledge and can, in principle, be reduced by increasing the amount of training samples or considering more appropriate models \citep{hullermeier_aleatoric_2021}. 

Model uncertainty relates to the uncertainty regarding the choice of model hypothesis space, i.e., uncertainty due to the model selection, and is often neglected \citep{hullermeier_aleatoric_2021}. The argument for ignoring model uncertainty is that in the case of neural networks, the assumption is made that their capacity is large enough that the suitable model is somewhere in this hypothesis space. However, architectural choices and regularization always drastically shrink this hypothesis space. Methods such as MC dropout can never account for this and hence cannot capture this source of uncertainty, as we demonstrate in our results.

Approximation epistemic uncertainty relates to the quality and the amount of data and the possibility of suboptimal paths followed by the learning algorithm \citep{hullermeier_aleatoric_2021}. Bayesian neural networks (BNN) and approximation variants, such as Monte Carlo (MC) dropout, quantify this approximation epistemic uncertainty. These can, in principle, also account for aleatoric uncertainty \citep{gal_dropout_2016, valdenegro-toro_deeper_2022}, however more in a frequentist way, and this is unrightfully often neglected.


% ...

\paragraph{Total predictive uncertainty}
In this work, we will focus on the total predictive uncertainty of a prediction, considering both aleatoric and epistemic uncertainty. The reasons for this choice are two-fold. First, considering uncertainty estimation in clinical machine learning applications where a diagnosis or prognosis is predicted, we believe that the total predictive uncertainty is of primary concern, i.e., a clinician is primarily interested in how much the prediction will deviate from the true outcome. Where this uncertainty comes from, due to model misspecification or inherent variability, is of second importance. Second, while a large part of the literature focuses on one of the specific sources of uncertainty or tries to disentangle these sources, it is essential to make clear that these uncertainties are also interrelated~\citep{valdenegro-toro_deeper_2022, banerji_clinical_2023}, i.e., the aleatoric part has an effect on the epistemic part, and the estimation of the aleatoric uncertainty contains epistemic uncertainty.

It is worth noting that while the focus here is on estimating the total predictive uncertainty, the underlying methodologies used (e.g., density estimation and ensemble modeling) could also be leveraged to quantify the aleatoric and epistemic uncertainty components if desired separately. This additional level of granularity may be helpful in specific applications where understanding the relative contributions of these uncertainty sources could provide additional insights~\citep{banerji_clinical_2023}. However, these would not give the finite sample guarantees of the conformal prediction framework.

\subsubsection{Uncertainty quantification in landmark localization}
Table \ref{tab:overview-works} provides an overview of the current literature on uncertainty quantification in landmark localization, highlighting the key distinctions and identifying existing gaps. A more elaborate literature review can be found in \ref{ap:uq_landmark}.
\input{overview_works}

Current works exhibit varied approaches to uncertainty quantification in landmark localization, such as MC dropout \citep{lee_automated_2020, drevicky_evaluating_2020, jafari_u-land_2022}, Gaussian processes \citep{schobs_bayesian_2023}, ensembles \citep{drevicky_evaluating_2020, schobs_uncertainty_2023}, test-time augmentation (TTA) \citep{ma_volumetric_2020}, maximum likelihood estimation (MLE) \citep{kumar_uglli_2019, avisdris_fetal_2021, mccouat_contour-hugging_2022}, or inferring uncertainty from the predicted heatmaps \citep{drevicky_evaluating_2020, payer_uncertainty_2020, thaler_modeling_2021, kwon_multistage_2021, mccouat_contour-hugging_2022, schobs_uncertainty_2023}. However, all of these works fail to reliably account for the total predictive uncertainty. Some only consider aleatoric uncertainty by performing maximum likelihood estimation \citep{kumar_uglli_2019, avisdris_fetal_2021, mccouat_contour-hugging_2022}. Others only try to model the approximation uncertainty and fail to consider aleatoric and model uncertainty \citep{drevicky_evaluating_2020, lee_automated_2020, schobs_uncertainty_2023}. Moreover, significant number of works also fail to evaluate the validity of their proposed approaches and uncertainty measures \citep{kumar_uglli_2019, kumar_luvli_2020,drevicky_evaluating_2020, lee_automated_2020, ma_volumetric_2020, payer_uncertainty_2020, thaler_modeling_2021, kwon_multistage_2021, jafari_u-land_2022, schobs_bayesian_2023}. 
Most existing methods are alo restricted to 2D landmark localization despite the growing use of 3D imaging as a standard in clinical practice. This shift from 2D to 3D imaging highlights the urgent need for uncertainty quantification techniques targeted to anatomical landmarks in 3D images. The existing literature should address this development more prominently, as only \citet{ma_volumetric_2020, payer_uncertainty_2020, thaler_modeling_2021} addresses 3D landmark localization.

This work tries to resolve the shortcomings by producing prediction regions with validity guarantees for a specified confidence level, ensuring reliable uncertainty estimates for 2D and 3D landmark predictions. This is achieved by introducing the conformal prediction framework for anatomical landmark detection to account for all sources of uncertainty and to give a notion of validity. We will also empirically validate the reliability/validity of these regions and those produced by other approaches.

\section{Uncertainty quantification for landmark localization}
\label{sec:uq-general}
\subsection{Notation and problem setup}
In this work, we denote our training data sequence as $(X_i, Y_i) = Z_i \in \mathcal{Z} = \mathcal{X} \times \mathcal{Y}$, where $X_i$ represents the input image from the input space $\mathcal{X}$. Where $\mathcal{X} \subseteq \mathbb{R}^{H \times W \times C}$ corresponds to 2D images of height $H$, width $W$, and $C$ channels, or $\mathcal{X} \subseteq \mathbb{R}^{D \times H \times W \times C}$ for 3D volumetric images with depth $D$. The response variable $Y_i$ represents a single anatomical landmark from the response space $\mathcal{Y}$ and $\mathcal{Y} \subseteq \mathbb{R}^d$, where $d \in {2,3}$ denotes the input image dimension. $\hat{C}_{\alpha}(X_{n+1};Z_{1:n})$ represents the estimated prediction interval with a targeted confidence level $1-\alpha$ for the image $X_{n+1}$ based on the previous examples $Z_{1:n} := Z_1, ..., Z_n$. For conciseness, we will often drop the conditioning on the previous examples in the prediction interval notation, i.e., $\hat{C}_{\alpha}(X_{n+1})$.

\subsection{Conformal prediction}
\label{sec:uq-general-cp}
Conformal prediction \citep{vovk_machine-learning_1999, saunders_transduction_1999, papadopoulos_inductive_2002, vovk_algorithmic_2022} is a framework that quantifies predictive uncertainty by outputting a prediction set with a related confidence level $1-\alpha$, which comes with a finite-sample coverage guarantee,
\begin{equation}
    \label{eq:marginal-guarantee}
    \mathbb{P}_{Z \sim P_{Z}}\left(Y_{n+1} \in \hat{C}_{\alpha}\left(X_{n+1};Z_{1:n}\right)\right) >= 1 - \alpha,
\end{equation}
under distribution-free assumptions, only requiring exchangeability of the training observations $Z_1, ..., Z_n$ and test object $Z_{n+1}$. The prediction sets in conformal prediction are formed by comparing nonconformity scores of examples to see how unusual a predicted label is. These scores measure the disagreement between the prediction and the actual target. In regression problems, the prediction sets are often referred to as prediction intervals in the case of a univariate response and prediction regions in the multivariate case, such as landmark localization.

One of the main advantages of conformal prediction is its model-agnostic nature, i.e., it can be applied to almost every machine learning model, as well as neural networks and anatomical landmark localization algorithms. However, the initial proposed approach, transductive (full) conformal prediction (TCP) \citep{saunders_transduction_1999}, is computationally expensive for most algorithms since it requires a refitting of the model for every possible label in the response domain. For regression problems, one refits the model for several grid points spanning the domain to make it somewhat computationally feasible, but still, for computationally expensive fitting procedures such as neural networks, this approach is practically infeasible. Therefore, we consider inductive (split) conformal prediction (ICP) \citep{papadopoulos_inductive_2002}, which is less computationally intensive than TCP. This decreased computational burden does not come for free because ICP cannot use the calibration set to optimize the point predictor, unlike TCP, which is likely more efficient. However, ICP allows conformal prediction to use neural networks. We provide the general procedure of ICP in Algorithm \ref{alg:general_ICP}.

\begin{algorithm}[H]
\caption{General Procedure of ICP}
\label{alg:general_ICP}
\begin{algorithmic}[1]
    \STATE {\bfseries Assumption:} Exchangeability of calibration set $Z_{n-m+1:n}$ and test example $Z_{n+1}$.
    \STATE {\bfseries Input:} Training data sequence $Z_{1:n}$, test object $X_{n+1}$, landmark localization model $\mathcal{L}$, and size of the calibration set $m$.
    \STATE The data sequence $Z_{1:n}$ is split into a proper training data set $Z_{1:n-m}$ and a calibration set $Z_{n-m+1:n}$.
    \STATE The proper training dataset $Z_{1:n-m}$ is used to train a landmark localization model $\mathcal{L}$.
    \STATE For each example $i$ in the calibration set $Z_{n-m+1:n}$, calculate the nonconformity score $S_{i}=s(X_i,Y_i)$, which relies on the landmark localization model.
    \STATE The nonconformity scores from the calibration set are sorted in descending order: $S^*_{1},..., S^*_{m}$.
    \STATE For each new test object $X_{n+1}$ and confidence level $1 - \alpha$, we return the prediction region 
    \begin{equation}
        \hat{C}_{\alpha}(X_{n+1}) = \left\{ y \in \mathcal{Y}: s(X_{n+1}, y) \leq S^*_{\lfloor \alpha (m+1) \rfloor} \right\}
    \end{equation}
\end{algorithmic}
\end{algorithm}

Validity and efficiency are the two main criteria to evaluate a confidence predictor \citep{fontana_conformal_2022}; they are also sometimes referred to as reliability and sharpness. Confidence predictions are always conservatively valid in conformal prediction and can be exactly valid using a smoothed conformal predictor \citep{vovk_algorithmic_2022, fontana_conformal_2022}. Since validity is always satisfied in conformal methods, efficiency should be optimized, i.e., the uncertainty related to predictions should be minimized. However, efficiency should never be optimized at the expense of validity. \citet{fontana_conformal_2022} phrase it quite clearly that \textit{"Validity is the priority: without it, the meaning of predictive regions is lost, and it becomes easy to achieve the best possible performance"}. This is the primary problem with the current state-of-the-art uncertainty quantification for landmark localization. Few of these works that discuss and propose methods for uncertainty quantification, evaluate the reliability of their methods (see Table \ref{tab:overview-works}), while that should be the starting point. For regression problems, efficiency in conformal prediction is often assessed with the width of the resulting prediction intervals, which we want to minimize to increase efficiency. However, in our landmark localization problem, the target has two or three dimensions. Thus, we also require multi-dimensional prediction regions, where the area or volume of the prediction region evaluates efficiency.

\subsubsection{Beyond marginal coverage}
Conformal prediction traditionally provides calibrated prediction intervals with marginal coverage. However, we often desire conditional coverage, which guarantees coverage across different data strata. However, unfortunately, true conditional coverage,
\begin{equation}
    \mathbb{P}_{Z \sim P_{Z}} \left( Y_{n+1} \in \hat{C}_{\alpha} \left(X_{n+1};Z_{1:n} \right) | X_{n+1} \right) >= 1 - \alpha,
\end{equation}
is impossible without making modeling assumptions or returning trivial intervals \citep{foygel_barber_limits_2021}. Several approaches have emerged to address this limitation and create more adaptive prediction intervals, guaranteeing something between finite sample marginal and true conditional coverage or giving asymptotic conditional guarantees.

Mondrian conformal prediction \citep{vovk_mondrian_2003} offers an intuitive solution by partitioning the example space using a measurable function that assigns each data point to a specific category. By applying the general conformal prediction framework to each category, one can obtain coverage guarantees conditional on these categories. However, this approach reduces the calibration set size, potentially increasing variance around the target coverage level.

A more sophisticated approach leverages nonconformity scores that incorporate conditional uncertainty estimates. Instead of using classical absolute residual errors, these methods use distributional predictions or statistics like quantiles or moments to create more adaptive prediction intervals. One can leverage the conditional estimate to obtain more conditionally relevant prediction intervals while still controlling the marginal coverage through the conformal prediction framework. This would make marginally valid prediction intervals more adaptive to the local variability of prediction error. A large region should indicate a high point prediction error, and a small region should indicate a low prediction error.

Several works have been proposed following this idea. Normalized nonconformity scores, proposed by \citet{papadopoulos_reliable_2011}, combine absolute error with a reciprocal uncertainty estimate. The uncertainty can be derived from separate models, MC dropout variance, or MLE. In regression problems, generally, these normalized scores consist of the product of the absolute error $|Y_i-\hat{f}(X_i)|$ and the reciprocal of a 1D uncertainty estimate $\hat{u}(X_i)$, $S_i = \frac{|Y_i-\hat{f}(X_i)|}{\hat{u}(X_i)}$.
\citet{romano_conformalized_2019} introduced a novel nonconformity score using conditional upper and lower quantile estimates. \citet{lei_efficient_2011, lei_distribution-free_2013} proposed to leverage the conditional density estimation $\hat{f}_{X_i}$, e.g., kernel density estimation, to construct the nonconformity score $S_i=-\hat{f}_{X_i}(Y_i)$ to create more adaptive prediction intervals. Similarly, \citet{chernozhukov_distributional_2021} proposed distributional conformal prediction, leveraging the probability integral transform in nonconformity score. \citet{sesia_conformal_2021} also leveraged the conditional distribution by using a histogram binning approach and nested prediction sets \citep{gupta_nested_2022}, allowing to get more efficient intervals for heavily skewed distributions. 

A recent innovative approach is the Regression to Classification Conformal Prediction (R2CCP) method proposed by \citet{guha_conformal_2023}. This technique transforms the regression problem into a classification task by binning the response domain and estimating a discrete probability distribution. The discrete distribution is then interpolated to create a continuous conditional distribution score, which is used to construct prediction intervals. The idea behind this approach is that it results in more stable training and is better at learning conditional expectations \citep{stewart_regression_2023, guha_conformal_2023}. Specifically, they propose to transform the response domain into $K$ bins covering the entire response domain $\mathcal{Y}$, resulting in a described label space $\hat{\mathcal{Y}} = \{\hat{Y}_1,..., \hat{Y}_K\}$, where $\hat{Y}_k$ represent the midpoints of the $k$th bin. They then fit a classifier with $K$ classes for estimating a discrete probability distribution $\hat{p}(\cdot|X_i)$. This distribution is transformed to continuous conditional distribution score $\hat{f}_{X_i}(y)$ by taking the linear interpolation of the discrete probability estimates,
\begin{equation}
    \hat{f}_{X_i}(y) = \gamma_k \hat{p}\left( \hat{Y}_{k}|X_i \right) + \left( 1-\gamma_k \right) \hat{p}\left(\hat{Y}_{k+1}|X_i \right), \quad \text{where } \hat{Y}_{k} \leq y < \hat{Y}_{k+1} \text{ and } \gamma_k = \frac{Y_{k+1}-y}{Y_{k+1}-Y_k}.
\end{equation}
This estimated continuous conditional distribution is then leveraged in the same manner as in \citet{lei_efficient_2011, lei_distribution-free_2013}, using the nonconformity scores $S_i=-\hat{f}_{X_i}(Y_i)$.

\subsubsection{Multi-output conformal prediction}
Most existing conformal prediction approaches primarily address univariate output regression problems despite landmark localization being inherently a multi-output regression task. Multi-output conformal prediction has emerged as a growing research area that provides valid prediction regions for multiple target variables.

Since the pioneering work by \cite{lei_distribution-free_2013}, researchers have proposed various approaches to extend single-output conformal prediction techniques. These approaches focus on critical aspects such as validity guarantees, efficiency, adaptivity, and computational feasibility. Our work categorizes multi-output conformal prediction approaches into five classes: statistically corrected, maximum nonconformity, copula-based, ellipsoidal, and density-based conformal prediction. A comprehensive review of these approaches is available in \ref{ap:cp_multi_output}.

This work discusses and benchmarks all these approaches except copula-based methods. Copula-based approaches require two calibration sets when employing split conformal prediction \citep{sun_copula_2024}, which can significantly reduce prediction efficiency. Furthermore, these methods produce hyperrectangular prediction regions aligned with the axes of the output dimensions, which may be suboptimal for coordinate predictions like landmark localization  (see Fig. \ref{fig:pred-region-2D} and \ref{fig:pred-region-3D}). Such regions tend to overcompensate for uncertainty, reducing the prediction regions' efficiency.

In this section, we review statistically corrected (Fig. \ref{fig:pred-region-2D} and \ref{fig:pred-region-3D}, top-left), maximum nonconformity (top-middle), and ellipsoidal (top-right) conformal prediction, focusing on their application to landmark localization. We also introduce two novel density-based approaches for multi-output regression in Section \ref{sec:uq-general-cp-mrcp}. The first is a method termed Multi-output Regression-as-Classification Conformal Prediction (M-R2CCP), which extends the R2CCP framework proposed by \citet{guha_conformal_2023}. Second, we introduce Multi-output Regression to Classification conformal prediction set to Region (M-R2C2R). This approach allows us to leverage conformal prediction approaches for classification problems, such as Adaptive Prediction Sets (APS) \citep{romano_classification_2020}. APS is a conformal prediction approach that specifically addresses the challenge of constructing prediction sets for classification problems that adapt to the difficulty of each instance, ensuring better efficiency (smaller sets) while maintaining coverage. As shown in the bottom panels of Fig. \ref{fig:pred-region-2D} and \ref{fig:pred-region-3D}, our approaches are way more flexible in the prediction regions they can represent and are, therefore, often more efficient.

\begin{figure}
    \centering
    \includegraphics[width=\textwidth]{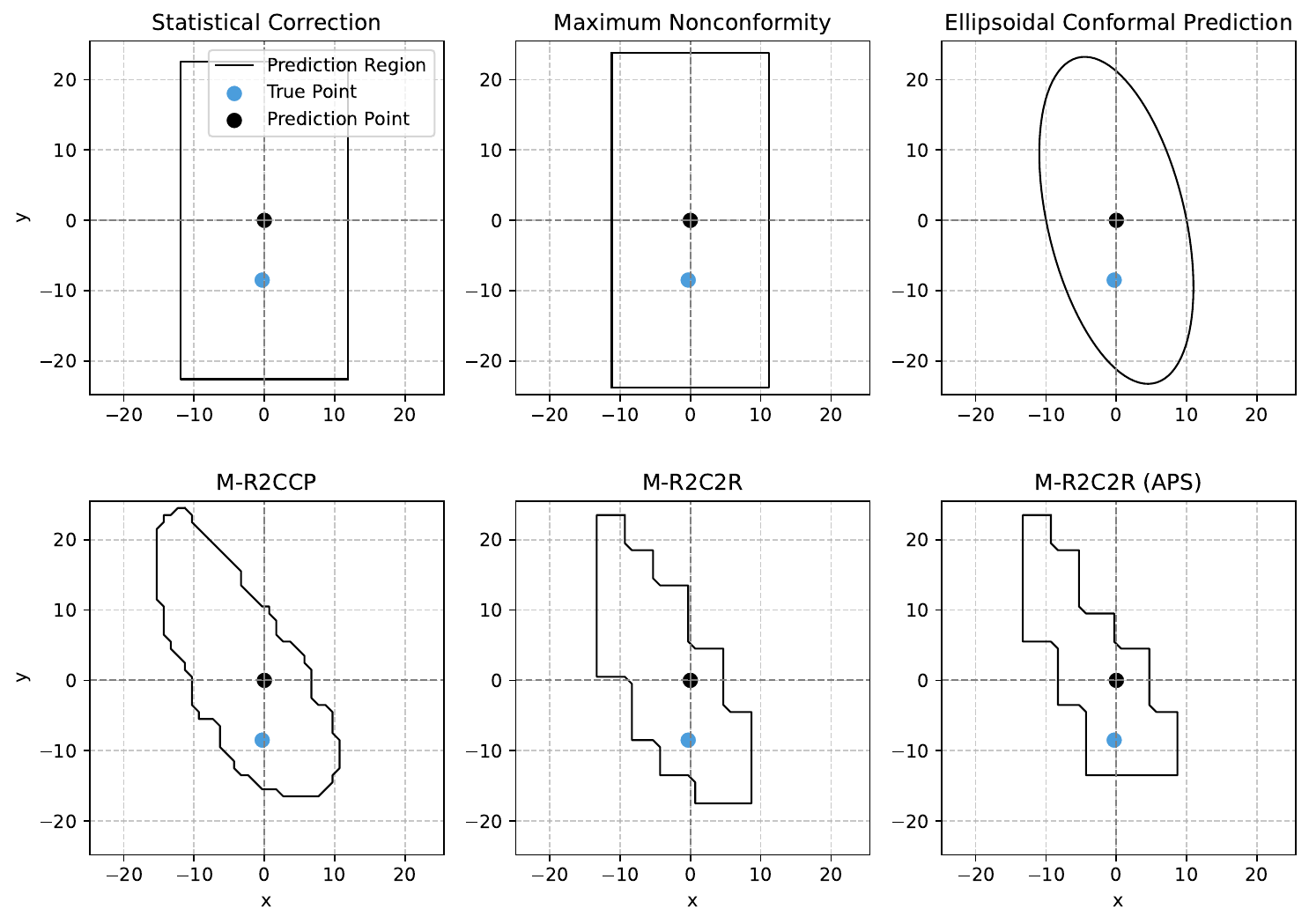}
    \caption{Examples of prediction region produced by different multi-output conformal prediction approaches (2D).}
    \label{fig:pred-region-2D}
\end{figure}

\begin{figure}
    \centering
    \includegraphics[width=\textwidth]{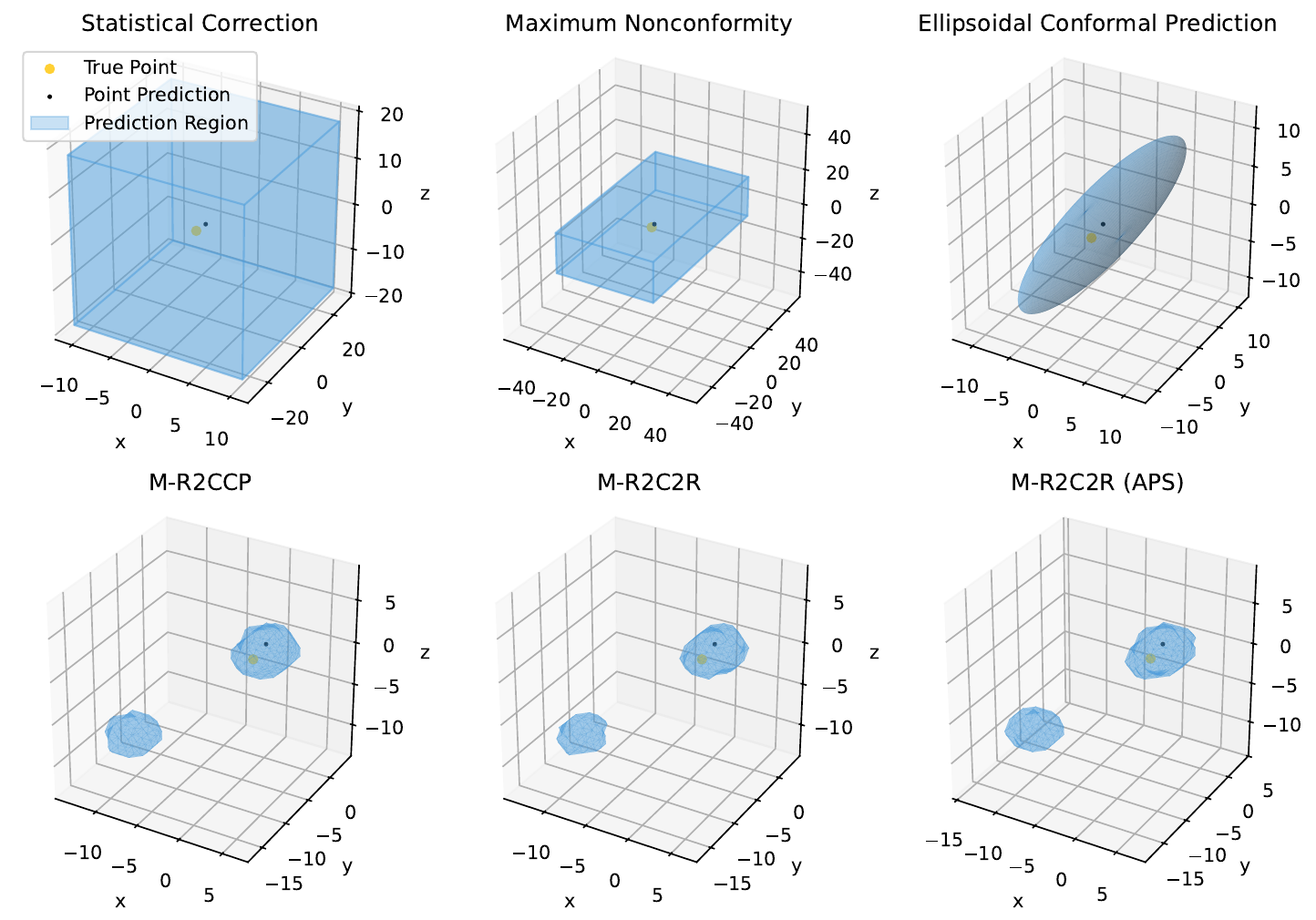}
    \caption{Examples of prediction region produced by different multi-output conformal prediction approaches (3D).}
    \label{fig:pred-region-3D}
\end{figure}

\paragraph{Statistical correction} In statistically corrected multi-output conformal prediction, predictions are initially generated separately for each output dimension. Global coverage across dimensions is achieved by applying dimension-wise corrections; see Algorithm \ref{alg:corrected_ICP}. 

While previous work \citep{messoudi_conformal_2020} used the Šidák correction, setting individual significance levels to $\alpha_t=1-\sqrt[d]{1-\alpha}$, which provides exact validity for independent scores but becomes conservative with positively dependent nonconformity scores and can result in invalid prediction regions when the scores are negatively dependent. Therefore, we opt for the Bonferroni correction $\alpha_t=1-\frac{\alpha}{d}$ for benchmarking purposes \citep{dunn_multiple_1961}. This correction comes with the marginal coverage guarantee (Equation \ref{eq:marginal-guarantee}) without any assumption about the relation of the nonconformity scores across different dimensions, see Proposition \ref{prop:bonferroni-ICP}. Note that one can use any nonconformity score; we use normalized scores that estimate uncertainty through variance, resulting in hyperrectangular prediction regions aligned with the axes (see Fig. \ref{fig:pred-region-2D} and \ref{fig:pred-region-3D}). We get the variance estimate of uncertainty through one of the approaches discussed in Section \ref{sec:other-approaches}.

\begin{algorithm}[H]
\caption{Statistically Corrected ICP}
\label{alg:corrected_ICP}
\begin{algorithmic}[1]
    \STATE {\bfseries Assumption:} Exchangeability of calibration set $Z_{n-m+1:n}$ and test example $Z_{n+1}$.
    \STATE {\bfseries Input:} Training data sequence $Z_{1:n}$, test object $X_{n+1}$, landmark localization model $\mathcal{L}$, size of the calibration set $m$, and output dimension $d$.
    \STATE The data sequence $Z_{1:n}$ is split into a proper training data set $Z_{1:n-m}$ and a calibration set $Z_{n-m+1:n}$.
    \STATE The proper training dataset $Z_{1:n-m}$ is used to train a landmark localization model $\mathcal{L}$.
    \FOR{$j \gets 1$ to $d$}
        \STATE For each example $i$ in the calibration set $Z_{n-m+1:n}$, calculate the nonconformity score $S_{i,j}=s(X_i,Y_{i,j})$ for dimension $j$, which relies on the landmark localization model.
        \STATE The nonconformity scores from the calibration set are sorted in descending order: $S^*_{1,j},..., S^*_{m,j}$
    \ENDFOR
    \STATE For each new test sample $X_{n+1}$ and confidence level $1 - \alpha$, we return the prediction region
    \begin{equation}
        \hat{C}_{\alpha}\left(X_{n+1}\right) = \left\{ y \in \mathcal{Y}: s\left(X_{n+1}, y_j \right) \leq S^*_{\lfloor \alpha_t (m+1) \rfloor, j} \text{ for } j = 1, \ldots, d \right\}
    \end{equation}
    where $\alpha_t$ is the adjusted significance level based on $\alpha$.
\end{algorithmic}
\end{algorithm}

\begin{restatable}{proposition}{bonferroniProp}
\label{prop:bonferroni-ICP}
Suppose that $Z_{1:n+1}$ are exchangeable, and $\alpha_t=\frac{\alpha}{d}$ (Bonferroni correction). Then, the prediction region defined by Algorithm \ref{alg:corrected_ICP} satisfies the marginal coverage guarantee in Equation \ref{eq:marginal-guarantee}.
\end{restatable}
\begin{proof}
    See \ref{ap:proofs}.
\end{proof}

\paragraph{Maximum nonconformity score}
The maximum nonconformity approach does not change anything to Algorithm \ref{alg:general_ICP}. It only uses a specific family of nonconformity scores,
\begin{equation}
    S_i := \max\left(S_{i,1}, \ldots,S_{i,d}\right),
\end{equation}
which allows for dealing with multi-output response values. The nonconformity score $S_{i,j}$ for example $i$ and dimension $j$, can be any nonconformity score used for single-output regression problems, such as the absolute error \citep{diquigiovanni_conformal_2022}, normalized absolute error \citep{dietterich_conformal_2022, diquigiovanni_conformal_2022}, or CQR \citep{dietterich_conformal_2022}. This work will use the normalized absolute error, like for the Bonferroni correction. Similarly, as for the Bonferroni correction, this approach will result in hyper-rectangle prediction regions aligned with the axes (see Fig. \ref{fig:pred-region-2D} and \ref{fig:pred-region-3D}).

\paragraph{Ellipsoidal conformal prediction}
Approaches \citep{johnstone_conformal_2021, messoudi_ellipsoidal_2022, henderson_adaptive_2024} using Mahalanobis-based nonconformity score enable multi-output response handling without changing Algorithm \ref{alg:general_ICP} by calculating scores as:
\begin{equation}
    S_i = \sqrt{\left(Y_i-\hat{f}(X_i) \right)^T \hat{\Sigma}^{-1}_i \left(Y_i-\hat{f}(X_i) \right)}.
\end{equation}
These methods produce ellipsoidal prediction regions that can be oriented in any direction based on the estimated covariance matrix. Uncertainty estimates can be derived from heatmaps, MC dropout, deep ensembles, or test-time augmentations (see Section \ref{sec:other-approaches}). Note that this can be seen as an extension of the normalized absolute residual scores to multivariate response domain.

\subsubsection{Multi-output regression-as-classification conformal prediction}
\label{sec:uq-general-cp-mrcp}
The existing single-output Regression-as-Classification Conformal Prediction (R2CCP) approach proposed by \citet{guha_conformal_2023} presents limitations when applied to multi-output problems, such as anatomical landmark localization. This work extends the R2CCP framework to address three key challenges:
\begin{enumerate}[leftmargin=2cm]
    \item[Challenge 1:] Handling multi-output response scenarios.
    \item[Challenge 2:] Allow leveraging conformal prediction approaches for the classification setting (e.g., adaptive prediction sets (APS) \citep{romano_classification_2020}) for regression problems.
    \item[Challenge 3:] Managing varying response domains across different examples (e.g., different image resolutions).
\end{enumerate}
To resolve these challenges, we introduce two key methodological advances. The first challenge leads to the development of Multi-output Regression to Classification Conformal Prediction (M-R2CCP). The solutions to the second and third challenges are formally established in Propositions \ref{prop:discrete2continuous} and \ref{prop:transformed-guarantee}, respectively. Combining these solutions results in Multi-output Regression to Classification conformal prediction set to Region (M-R2C2R).

\paragraph{Challenge 1: Addressing multi-output responses} We propose binning the response space into equally sized hyperrectangular regions to handle multi-output responses. While pixel-level binning is conceptually straightforward for medical images, it quickly becomes computationally intractable. Instead, we recommend creating bins that encompass multiple pixels or voxels. Like the univariate approach, we fit a classifier with $K$ classes for estimating a discrete probability distribution $\hat{p}(\cdot|X_i)$. We then transform this discrete distribution into a continuous conditional distribution score $\hat{f}_{X_i}(y)$ using linear interpolation:
\begin{equation}
\label{eq:n-interpolation}
\hat{f}_{X_i}(y) = \sum_{(k_1,\ldots,k_d) \in \mathcal{I}} \left[ \prod_{j=1}^{d} \gamma_{k_j,j} \right] \hat{p}\left(\hat{Y}_{(k_1,\ldots,k_d)}|X_i\right),
\end{equation}
where $\mathcal{I}$ is the set of all interpolation indices, $\gamma_{k_j,j} = \frac{Y_{k_j+1,j} - y_j}{Y_{k_j+1,j} - Y_{k_j,j}}$ for $\hat{Y}_{k_j,j} \leq y_j < \hat{Y}_{k_j+1,j}$. In Equation \ref{eq:n-interpolation}, the summation covers all hypercube vertices involved in the n-linear interpolation and $\prod_{j=1}^{d} \gamma_{k_j,j}$ computes the interpolation weight for each vertex. The nonconformity scores are computed as $S_i=-\hat{f}_{X_i}(Y_i)$. Algorithm \ref{alg:M-R2CCP-ICP} provides a more detailed implementation of the Multi-output Regression to Classification Conformal Prediction (M-R2CCP). 
While various interpolation methods can transform discrete bins into continuous responses, the specific approach may impact interval efficiency. Future research could systematically investigate the performance of different interpolation techniques.

One minor problem with (M-)R2CCP is it requires an evaluation of the entire target domain to construct the prediction interval (Algorithm \ref{alg:M-R2CCP-ICP}, line 9). This is often resolved by performing a grid search evaluation on the target domain.

\begin{algorithm}[H]
\caption{Multi-output Regression to Classification Conformal Prediction (M-R2CCP).}
\label{alg:M-R2CCP-ICP}
\begin{algorithmic}[1]
    \STATE {\bfseries Assumption:} Exchangeability of calibration set $Z_{n-m+1:n}$ and test example $Z_{n+1}$.
    \STATE {\bfseries Input:} Training data sequence $Z_{1:n}$, test object $X_{n+1}$, landmark localization model $\mathcal{L}$, size of the calibration set $m$, and output dimension $d$.
    \STATE {\bfseries Hyperparameters:} number of bins $K_j > 1$ for each dimension $j$
    \STATE For each dimension $j$ of the output space $\mathcal{Y}$, discretize the output space dimension into $K_j$ equidistant bins with midpoints $\{\hat{Y}_{1,j}, \ldots, \hat{Y}_{K_j,j}\}$
    \STATE The data sequence $Z_{1:n}$ is split into a proper training data set $Z_{1:n-m}$ and a calibration set $Z_{n-m+1:n}$.
    \STATE Find a discrete probability distribution $\hat{p}(\cdot|X_i)$ by optimizing the landmark localization model $\mathcal{L}$ on the proper training dataset $Z_{1:n-m}$. \comment{For example, by using the one-hot heatmap approach.}
    \STATE For each example $i$ in the calibration set $Z_{n-m+1:n}$, calculate the nonconformity score $S_{i}= -\hat{f}_{X_i}(Y_i)$ (see Equation \ref{eq:n-interpolation}).
    \STATE The nonconformity scores from the calibration set are sorted in descending order: $S^*_{1},..., S^*_{m}$.
    \STATE For each new test sample $X_{n+1}$ and confidence level $1 - \alpha$, we return the prediction region 
    \begin{equation}
        \hat{C}_{\alpha}(X_{n+1}) = \left\{ y \in \mathcal{Y}: -\hat{f}_{X_{n+1}}(Y_{n+1}) \leq S^*_{\lfloor \alpha (m+1) \rfloor}\right\}
    \end{equation}
\end{algorithmic}
\end{algorithm}

\paragraph{Challenge 2: Classification conformal prediction for regression} 
Rather than treating the multi-output regression problem through traditional grid search methods, we propose leveraging classification conformal prediction approaches. This strategy allows us to leverage advanced classification conformal prediction techniques that generate more adaptive prediction sets.

Our Multi-output Regression to Classification conformal prediction set to Region (M-R2C2R) approach transforms the regression problem into a classification task. By generating prediction sets through classification conformal prediction and converting these discrete sets to continuous prediction regions, we can utilize sophisticated methods like adaptive prediction sets (APS) \citep{romano_classification_2020}.

Classification conformal prediction approaches offer several advantages for regression problems. They can potentially produce more example-driven prediction sets that adapt more flexibly to the underlying data distribution; in some cases, we can also reduce the computation burden of the M-R2CCP gird search. 

The M-R2C2R approach outlined in Algorithm \ref{alg:M-R2C2R} and its validity of the generated prediction regions is guaranteed by Propositions \ref{prop:discrete2continuous}. The transformation of the prediction region in M-R2C2R is handled next. 

\begin{restatable}{proposition}{dtocProp}
\label{prop:discrete2continuous}
    Suppose a transformation $g: \mathcal{Y}  \rightarrow \mathcal{P(Y')}$, where $g(y)$ maps each $y \in \mathcal{Y}$ to a set in $\mathcal{Y'}$, and assuming we have a marginal coverage guarantee on $Y_{n+1}  \subseteq \mathcal{Y}$, then the marginal coverage guarantee holds under the transformation of the response and prediction set,
    \begin{equation}
        \label{eq:dtocProp-marginal-guarantee}
        \mathbb{P}_{Z \sim P_{Z}}\left(g(Y_{n+1}) \in g\left(\hat{C}_{\alpha}(X_{n+1};Z_{1:n})\right)\right) >= 1 - \alpha,
    \end{equation}
    provided that the transformed prediction set is constructed as
    \begin{equation} 
        g\left(\hat{C}_{\alpha}(X_{n+1};Z_{1:n})\right) = \bigcup_{y \in \hat{C}_{\alpha}(X_{n+1}; Z_{1:n})} g(y).
    \end{equation}
\end{restatable}
\begin{proof}
    See \ref{ap:proofs}.
\end{proof}

\paragraph{Challenge 3: Addressing varying response domains} We propose a domain transformation approach that standardizes input and response dimensions. We prove that we maintain marginal coverage guarantees by converting all examples to a fixed dimension (e.g., $512 \times 512$), computing nonconformity scores, and subsequently rescaling prediction regions. This method seamlessly integrates with existing anatomical landmark localization pipelines, which typically perform inference on rescaled images. The procedure still comes with the marginal coverage guarantee since it is a monotonic transformation of the response domain; see Proposition \ref{prop:transformed-guarantee}. Note that it can also be applied to previously described conformal prediction approaches.

\begin{restatable}{proposition}{transGuarProp}
\label{prop:transformed-guarantee}
    Suppose the transformation of the response domain is denoted by $g: \mathcal{Y}  \rightarrow \mathcal{Y}'$, and assuming we have a marginal coverage guarantee (see Equation \ref{eq:marginal-guarantee}) on $Y_{n+1} \subseteq \mathcal{Y}$, then the marginal coverage guarantee holds under the transformation of the response and prediction set,
    \begin{equation}
        \label{eq:transGuarProp-marginal-guarantee}
        \mathbb{P}_{Z \sim P_{Z}}\left(g(Y_{n+1}) \in g\left(\hat{C}_{\alpha}(X_{n+1};Z_{1:n})\right)\right) >= 1 - \alpha,
    \end{equation}
    when $g$ is a monotonic OR invertible transformation.
\end{restatable}

\begin{proof}
     See \ref{ap:proofs}.  
\end{proof}

\begin{algorithm}[H]
\caption{Multi-output Regression to Classification conformal prediction set to Region (M-R2C2R).}
\label{alg:M-R2C2R}
\begin{algorithmic}[1]
    \STATE {\bfseries Assumption:} Exchangeability of calibration set $Z_{n-m+1:n}$ and test example $Z_{n+1}$.
    \STATE {\bfseries Input:} Training data sequence $Z_{1:n}$, test object $X_{n+1}$, landmark localization model $\mathcal{L}$, size of the calibration set $m$, and output dimension $d$.
    \STATE {\bfseries Hyperparameters:} number of bins $K_j > 1$ for each dimension $j$
    \STATE Transform the output space $\mathcal{Y}$ to $\mathcal{Y^*}$. \comment{E.g., transform all images and corresponding coordinates to a $512\times 512$ resolution. (Challenge 3)}
    \STATE For each dimension $j$ of the output space $\mathcal{Y}$, discretize the output space dimension into $K_j$ equidistant bins with midpoints $\{\hat{Y}_{1,j}, \ldots, \hat{Y}_{K_j,j}\}$
    \STATE The data sequence $Z_{1:n}$ is split into a proper training data set $Z_{1:n-m}$ and a calibration set $Z_{n-m+1:n}$.
    \STATE Find a discrete probability distribution $\hat{p}(\cdot|X_i)$ by optimizing the landmark localization model $\mathcal{L}$ on the proper training dataset $Z_{1:n-m}$. \comment{For example, by using the one-hot heatmap approach.}
    \STATE Create a discrete prediction set $C_\alpha^{'}(X_{n+1})$ using $\hat{p}(\cdot|X_i)$, the calibration set  $Z_{n-m+1:n}$, test object $X_{n+1}$, and conformal prediction for categorical labels.
    \STATE Transform the discrete prediction set $C_\alpha^{'}(X_{n+1})$  to the original response domain creating the continuous prediction region $C_\alpha(X_{n+1})$.  \comment{If a bin is present in $C_\alpha^{'}(X_{n+1})$ then all response values corresponding to that bin will be present in $C_\alpha^*(X_{n+1})$.}
    \STATE Transform prediction region $C_\alpha^*(X_{n+1})$ to the original output space $\mathcal{Y}$ resulting in prediction region $C_\alpha(X_{n+1}).$ \comment{(Challenge 3)}
\end{algorithmic}
\end{algorithm}

\subsection{Other approaches}
Several non-conformal techniques dominate the anatomical landmark localization literature (see Table \ref{tab:overview-works}). These approaches serve as benchmarks in our study, and we also incorporate their uncertainty estimates in some of our proposed conformal prediction approaches.

\label{sec:other-approaches}
\subsubsection{Test time augmentation}
Test time augmentation (TTA) randomly samples multiple augmentations of the input image during inference. It performs prediction on these samples, resulting in sample predictions that can be represented by sample statistics or empirical distributions. These transformations can be spatial, such as affine or elastic transformations, or intensity transformations, such as adding Gaussian noise or applying histogram shifts. TTA is also often used to improve accuracy and robustness. 

\citet{wang_aleatoric_2019} provided a mathematical framework for TTA in image segmentation, comparing it with Monte Carlo dropout. While they argue that TTA can quantify image-related aleatoric uncertainty, the approach likely does not capture all sources of aleatoric uncertainty, such as label noise arising from inter- and intra-rater variability. The authors additionally proposed a combined approach integrating TTA and Monte Carlo dropout to quantify the more exhaustive notions of uncertainty, i.e., predictive uncertainty.

\subsubsection{Monte Carlo dropout} Introduced by \citet{gal_dropout_2016}, Monte Carlo (MC) dropout is a method for modeling epistemic uncertainty in deep neural networks by leveraging dropout layers. Typically, dropout is used only during training as a regularization technique. In MC dropout, multiple forward passes are performed through the network with dropout enabled during inference, allowing uncertainty estimation through the variance of these passes. Similar to TTA, it can also be used to improve accuracy and robustness.

\citet{gal_dropout_2016} demonstrated that dropout approximately integrates over model weights and mathematically equates a neural network with dropout to a deep Gaussian process. In our work, we extend this approach by estimating the covariance of response output for our multivariate problem, which can be derived directly from the MC samples.

\subsubsection{Deep ensemble}
\citet{lakshminarayanan_simple_2017} proposed deep ensembles as a method for predictive uncertainty estimation. This approach involves creating an ensemble of neural networks with identical architecture but different weight initializations. Uncertainty is estimated by generating multiple prediction samples and analyzing their distribution.
While primarily addressing epistemic uncertainty in regression contexts, \citet{fort_deep_2020} argued that deep ensembles generally outperform MC dropout due to more decorrelated inference models.

\subsubsection{Temperature scaling} 
Temperature scaling, introduced by \citet{guo_calibration_2017}, is a calibration technique applicable to multi-class problems, specifically useful for one-hot heatmap approaches \citep{mccouat_contour-hugging_2022}. The method introduces a single parameter $\tau$ (temperature) to calibrate heatmap probabilities by dividing each channel pixel before softmax activation.
Optimal $\tau$ is determined using a calibration set to minimize the negative log-likelihood of the one-hot encoded heatmap.

\subsubsection{Heatmap derived uncertainty}
Research by \citet{payer_uncertainty_2020, thaler_modeling_2021} has explored uncertainty derivation by fitting multivariate Gaussian distributions to predicted heatmaps using least-squares curve fitting. In our work, we derive anatomical landmarks and their corresponding covariance matrix by calculating the weighted sample mean and covariance. 

In this work we also introduce another approach, the Naive Regression-to-Classification-to-Regression (Naive R2C2R) approach. Naive R2C2R uses the heatmap as a probability distribution. It works similarly as the APS approach \citep{romano_classification_2020}; however, it does so without an adjustment to control coverage and thus consequently assumes an "oracle" distribution, i.e., a perfectly calibrated distribution, and thus does not come with any coverage guarantees. In practice, it ranks pixel locations in descending order based on predicted probabilities and then progressively adds these pixels to the prediction regions until the cumulative probability of that region meets the desired confidence level. This method can also be enhanced with temperature scaling and/or pixel averaging through deep ensembles, MC dropout, or TTA.

\section{Datasets}
\label{sec:datasets}
\subsection{ISBI 2015 cephalometric grand challenge}
The ISBI 2015 Automatic Cephalometric X-Ray Landmark Detection Challenge selected 19 landmarks commonly used in clinical practice for evaluating the ability of automatic detection of these landmarks \citep{wang_evaluation_2015}. These assessments/landmarks suffer from relatively high intra- and inter-observer errors \citep{kamoen_clinical_2001}. The dataset consists of 400 annotated radiographs; we use the same splitting procedure as in previous works \citep{lindner_fully_2016, zhong_attention-guided_2019, thaler_modeling_2021}, i.e., splitting the data into four equally sized folds, and only use the annotations of the junior annotator. This follows previous works and mitigates a systemic shift between folds annotations \citep{lindner_fully_2016, zhong2019attention, thaler_modeling_2021}. Two folds are used as the proper training set, one as a calibration set and one as a test set to evaluate the proposed approaches.


\subsection{Canine hip dysplasia (CHD) landmark dataset}
A second dataset used to evaluate the proposed approach in this work is a private dataset consisting of annotated canine pelvis X-rays where each radiograph has $12$ annotated landmarks (see Fig. \ref{fig:canine-cp-example}), which we refer to as the Canine Hip Dysplasia (CHD) dataset. These landmarks are used to derive radiograph measurements, which in turn are used to assess the development of hip dysplasia, one of dogs' most prevalent diseases. The dataset consists of $472$ radiographs, of which $310$ are part of the training set and $162$ of the test set. The $310$ radiographs are split into a proper training set with $207$ radiographs and a calibration set with $103$ radiographs.

\subsection{Mandibular molar landmarking dataset}
Finally, we also use the Mandibular Molar Landmarking (MML) dataset \citep{he_anchor_2024}, which consists of 658 annotated CT volumes of the skull. The CT volumes are transformed to a uniform scale of $512\times 512 \times 256$. Multiple junior and senior clinicians annotated the CT volumes. It includes $14$ mandibular landmarks, specifically, $4$ crowns, $8$ roots of the second and third mandibular molars, and $2$ cups of cuspids. Note that the number of landmarks in a single CT volume is arbitrary, i.e., some landmarks are absent in the images. Since this work only focuses on landmark localization, we filter out all CT volumes where not all landmarks are present, reducing the dataset size to $399$. For evaluating uncertainty quantification approaches, we split the data into a proper training dataset of $199$ images, and a calibration and test set of each $100$ images.
For benchmarking the landmark localization model against the state-of-the-art, we use the evaluation procedure proposed by \citet{he_anchor_2024}. Although we propose to correct their deterministic evaluation procedure. The authors did not correct the pixel spacing for the scaling that is possibly performed during their center crop transformation. The pixel spacing should be dived by the applied scaling to get a correct point error in metric units (e.g., mm). This adjustment gives a slightly higher point error for the evaluation dataset, as seen in Tables \ref{tab:mml-det-benchmark} and \ref{tab:mml-det-adj-benchmark} presenting deterministic performance, where the only difference is the scale adjustment.

\section{Experiments}
\label{sec:experiments}
\subsection{Landmark localization}
This work emphasizes model-agnostic uncertainty quantification in landmark localization. However, we employ two distinct approaches to landmark localization to assess the effectiveness of different uncertainty quantification approaches. 

First, we utilize the SpatialConfigurationNet (SCN) model, which incorporates learned homoscedastic aleatoric uncertainty for each landmark, as introduced by \citet{payer_uncertainty_2020, thaler_modeling_2021}. The SCN is used as a landmark localization model across all 2D datasets presented in Section \ref{sec:datasets}. We use the same loss function, hyperparameters, and optimization strategies described in \cite{payer_uncertainty_2020}.

Second, we implement the one-hot approach, introduced in \citet{mccouat_contour-hugging_2022}, which transforms landmark localization into a classification task using one-hot encoded heatmaps. For 2D landmark localization, we propose using a UNet architecture that employs a ResNet-34 encoder pre-trained on ImageNet, with a depth of 5 layers. The decoder features channels of size 256, 128, 64, 32, and 32, respectively. A softmax activation is applied at the output layer to ensure proper probability distribution. This approach is ideal for the M-R2CCP and M-R2C2R approaches proposed in this work.

In this work, we suggest adapting the one-hot technique for the 3D landmark task. The foundational principle remains unchanged: we utilize a 3D UNet that integrates an EfficientNet-B0 encoder with MedicalNet's pre-trained weights, featuring a depth of 5 layers. The decoder comprises of channels sized 256, 128, 64, 32, and 16. A softmax activation is implemented at the output layer to ensure a probability distribution summing to one. This adaptation demonstrates strong performance, surpassing other methods assessed on the MML datasets (see Table \ref{tab:mml-sota}). 

For both methodologies and across all datasets, we apply a data augmentation strategy that includes random spatial transformations (affine transformations) and intensity transformations, such as adding Gaussian noise, scaling, gamma adjustments, and nonlinear transformations of the image histogram.

We also assess how sampling-based uncertainty quantification methods — like TTA, MC dropout, and deep ensembles — enhance predictive performance. Our evaluation consists of two strategies: the first involves averaging the heatmaps at the pixel level, leading to a more robust heatmap that is subsequently decoded into predictions of anatomical landmarks. The second method aligns more with uncertainty quantification approaches, as it decodes the heatmap for each individual sample and averages the landmark predictions across samples. This approach is denoted as LS. Results are presented in Tables \ref{tab:isbi-det}, \ref{tab:chd-det}, and \ref{tab:mml-det-benchmark} in \ref{ap:det-performance}). We observed that pixel-level sampling significantly increases performance across all datasets compared to the baselines and the LS approaches.

\input{sota_mml}

\subsection{Empirical results of uncertainty quantification approaches}
We assess the various prediction regions of the uncertainty quantification methods by analyzing three criteria: validity, efficiency, and adaptivity. The evaluation is conducted on the two 2D datasets and the 3D dataset to ensure diversity in our assessments. 

\begin{figure}
    \centering
    \includegraphics[width=\textwidth]{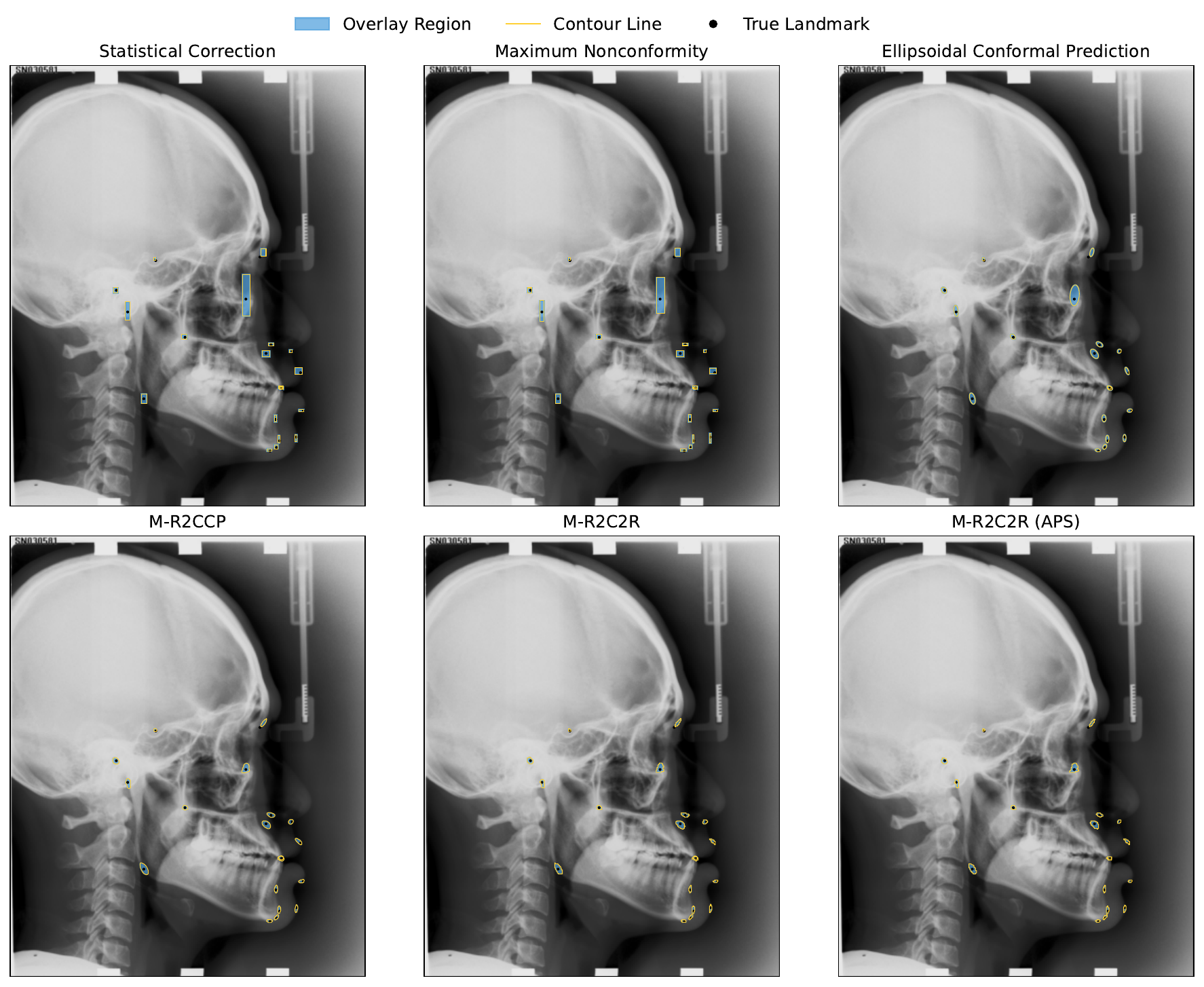}
    \caption{Cephalogram example of conformal prediction regions.}
    \label{fig:ceph-cp-example}
\end{figure}

\subsubsection*{Coverage of prediction regions}
We evaluate the validity of this work by measuring the marginal coverage of the prediction regions of different landmarks on different datasets. The marginal coverage for the ISBI 2015 (2D), CHD (2D° and MML (3D) datasets are presented in box plots in respectively Fig. \ref{fig:coverage-isbi2015}, \ref{fig:coverage-chd} and \ref{fig:coverage-mml}. We immediately observe that the frequently used sample inference approaches (MC dropout, deep ensemble, and TTA), combined with the normality assumption of the predictive uncertainty, drastically underestimate the prediction uncertainty. This can be attributed to the inability of the approaches to account for all sources of uncertainty, as discussed in the related work (see Section \ref{sec:related-work-uq}). The heatmap-derived approach, proposed by \cite{thaler_modeling_2021}, seems to account a bit more for the predictive uncertainty, but it still systematically underestimates the total predictive uncertainty.
To resolve this underestimation of total predictive uncertainty, we propose the conformal prediction framework to give us a finite-sample probabilistic coverage guarantee. The coverage results, presented in Fig. \ref{fig:coverage-isbi2015}, \ref{fig:coverage-chd}, and \ref{fig:coverage-mml}, show that the conformal prediction framework gives us a valid prediction region. We note that across all three datasets, as expected, the Bonferroni correction and max nonconformity approaches give slightly more conservative prediction regions. 

When temperature scaling is applied, the Naive R2CR approaches perform reasonably well on the ISBI 2015 and CHD datasets, meaning that the temperature scaling ensures that probability distribution resembles the Oracle predictive distribution. However, temperature scaling cannot guarantee this, and this can be clearly seen for MML (3D) dataset, as the prediction regions are significantly undercover even with temperature scaling, pressing the need for the conformal prediction framework. It is interesting to note that, in general, the pixel-averaging approach using deep ensembles generally increases the validity of the prediction regions. Thus, in general, the pixel-averaging approach using deep ensembles increases the landmark localization approach and calibration of the pixel probabilities in the one-hot approach.

\begin{figure}
    \centering
    \includegraphics[width=\textwidth]{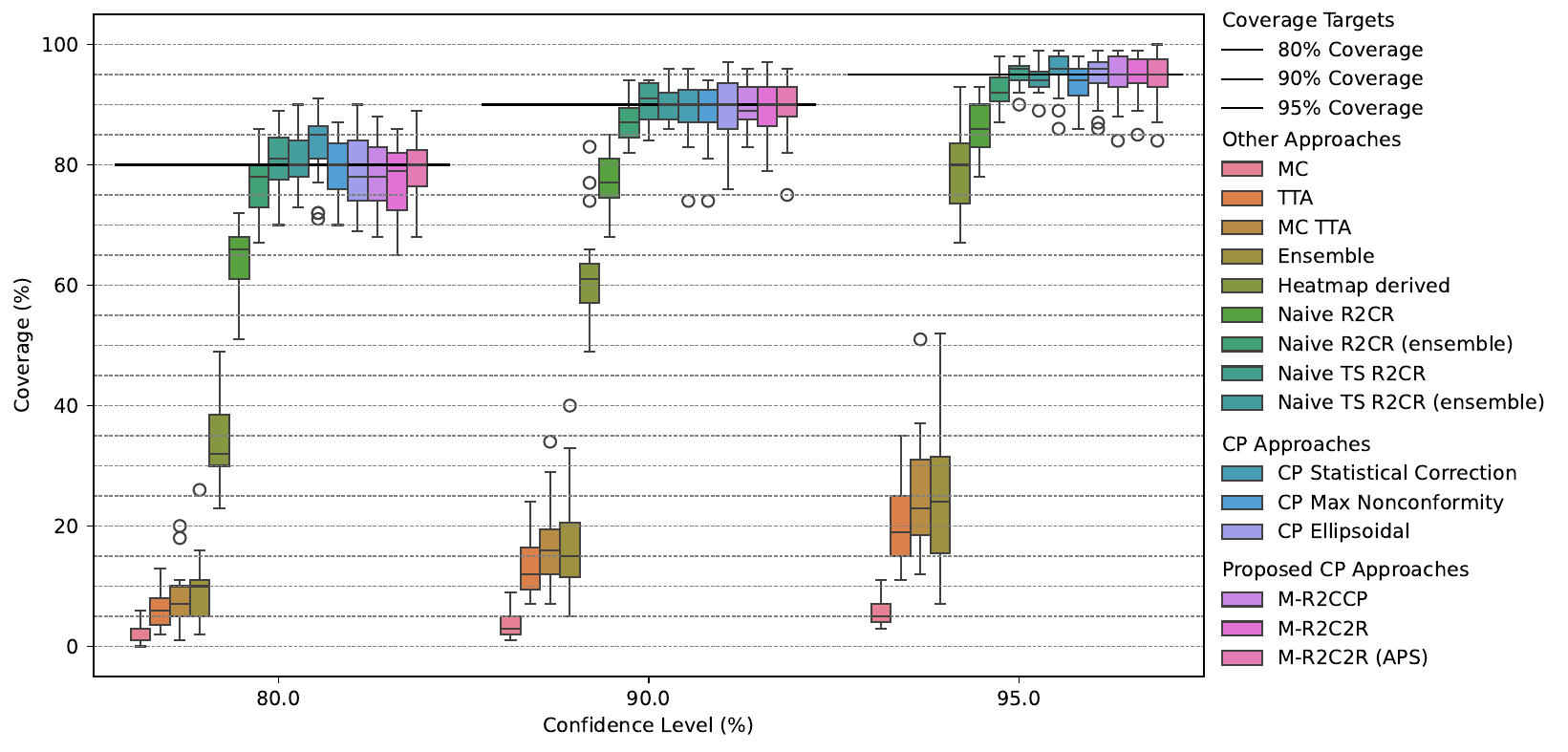}
    \caption{Box plot showing the empirical coverage of prediction regions for different landmarks (19 landmarks), evaluating various uncertainty quantification approaches at different confidence target levels on the ISBI 2015 (2D) dataset test set.}
    \label{fig:coverage-isbi2015}
\end{figure}

\begin{figure}
    \centering
    \includegraphics[width=\textwidth]{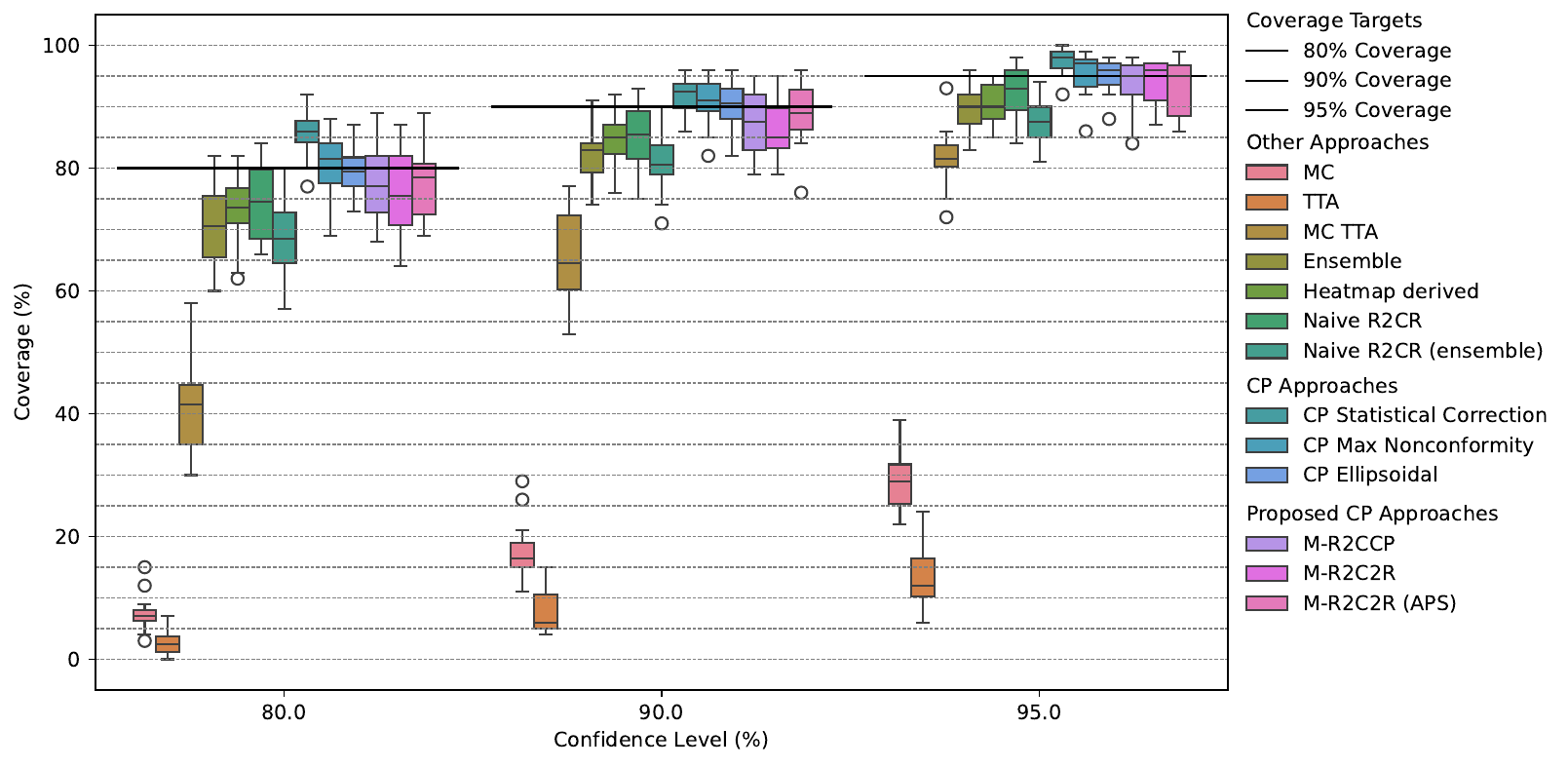}
    \caption{Box plot showing the empirical coverage of prediction regions for different landmarks (14 landmarks), evaluating various uncertainty quantification approaches at different confidence target levels on the MML (3D) dataset test set.}
    \label{fig:coverage-mml}
\end{figure}

\input{uq_results_isbi2015}
\input{uq_results_mml}

\subsubsection*{Efficiency and adaptivity of prediction regions}
In the evaluation of prediction regions, validity serves as the primary criterion, as it underpins the interpretability and reliability of prediction regions. Additionally, it is easy to get good efficiency when you compromise on validity. Consequently, our efficiency analysis focuses exclusively on conformal prediction approaches, which provide theoretical finite-sample validity guarantees. We assess efficiency through the measurement of prediction region area or volumes, utilizing both mean and median metrics to account for the variability in region size across different images, a variation that may appropriately reflect underlying predictive uncertainty. The results of our experiments are presented in Table \ref{tab:isbi2015-uq}, \ref{tab:chd-uq}, and \ref{tab:mml-uq}. Additionally, to give an idea of the prediction regions, Fig. \ref{fig:ceph-cp-example} shows the different conformal prediction regions on a cephalogram from the ISBI 2015 dataset.

We observe that the Bonferroni correction and max nonconformity approaches generate the largest prediction regions. This can be attributed to two factors: their inherently conservative nature in region estimation and, more significantly, their constraint to axis-aligned hyperrectangular prediction region. In contrast, the Mahalanobis nonconformity scores offer enhanced flexibility through ellipsoidal prediction regions, whose dimensions and orientations are determined by the predicted covariance matrix, resulting in improved efficiency.

Our M-R2CCP and M-R2C2R approaches demonstrate superior performance, capable of representing non-convex regions (as illustrated in Fig. \ref{fig:pred-region-2D} and \ref{fig:pred-region-3D}), and consistently achieve the highest efficiency metrics. Notably, the various implementations of M-R2CCP and M-R2C2R exhibit comparable performance levels.

The relationship between prediction error and region size is visualized through scatter plots in Fig. \ref{fig:isbi2015-adaptivity}, \ref{fig:chd-adaptivity}, and \ref{fig:mml-adaptivity}. While the APS variant of M-R2C2R exhibits enhanced adaptivity—aligning with its design objectives—the practical significance of these adaptivity differences requires further investigation. It is noteworthy that conformal ellipsoid prediction regions can inherit adaptivity characteristics from sampling-based uncertainty quantification approaches (MC dropout, TTA, or deep ensembles) by leveraging their covariance matrices.

Furthermore, it is important to note that the high Spearman correlation observed in the MC dropout approach for the MML dataset (Table \ref{tab:mml-uq}) can be primarily attributed to elevated prediction errors resulting from the landmark sampling methodology.

\section{Conclusion}
\label{sec:conclusion}
Through extensive empirical evaluation across datasets in both 2D and 3D domains, we demonstrate that commonly used uncertainty quantification approaches—such as MC dropout, deep ensembles, and test-time augmentation; systematically underestimate the total predictive uncertainty when combined with normality assumptions. To address these limitations, this work introduces the conformal prediction framework for reliability quantifying uncertainty in anatomical landmark localization, addressing a critical gap in the literature. 

Besides introducing existing multi-output conformal prediction approaches for uncertainty quantification in landmark localization, we introduce M-R2CCP and its variant M-R2C2R, two novel approaches that guarantee finite-sample validity while maintaining computational efficiency. Our experimental results demonstrate that M-R2CCP and M-R2C2R consistently outperform existing multi-output conformal prediction approaches regarding validity and efficiency. Unlike conventional methods that produce axis-aligned hyperrectangular or ellipsoidal regions, our approaches can generate flexible, non-convex prediction regions that better capture the underlying uncertainty structure of landmark predictions. Future work should investigate if this approach is as successful for other multi-output regression problems.

In general, this work highlights the importance of rigorous validity guarantees in clinical applications. The proposed conformal prediction framework represents a significant step forward in reliable uncertainty estimation for anatomical landmark localization, providing clinicians with trustworthy confidence measures for their diagnoses. Future work could explore extending these methods to handle missing landmarks, investigating alternative interpolation techniques for discrete-to-continuous mapping, and designing automatic decision-making algorithms that leverage the prediction regions. 


\section*{Data Availability}
All data is openly available except the CHD dataset, which is a proprietary dataset. The code related to this work can be found here: \hyperlink{https://github.com/predict-idlab/landmark-uq}{https://github.com/predict-idlab/landmark-uq}. The methods presented in this work will also be implemented in the \texttt{landmarker} package.

\section*{Funding sources}
Jef Jonkers is funded by the Research Foundation Flanders (FWO, Ref. 1S11525N).

\bibliographystyle{unsrtnat} 
\bibliography{references}

\appendix
\section{Proofs}
\label{ap:proofs}

\bonferroniProp*
\begin{proof}
    The validity of each ICP follows from Proposition 4.1 in \citet{vovk_algorithmic_2022}, and by applying Boole's inequality of the individual prediction intervals,
    \begin{align}
        \mathbb{P}_{Z \sim P_{Z}}\left(Y_{n+1} \not\in \hat{C}_{\alpha}(X_{n+1};Z_{1:n})\right) &\leq \alpha \\
        \mathbb{P}_{Z \sim P_{Z}}\left(\bigcup_{j=1}^d Y_{n+1,j} \not\in \hat{C}_{\alpha_t,j}\left(X_{n+1};Z_{1:n}\right)\right) &\leq \alpha \\
        \sum_{j=1}^d \mathbb{P}_{Z \sim P_{Z}}\left(Y_{n+1,j} \not\in \hat{C}_{\alpha,j}\left(X_{n+1};Z_{1:n}\right)\right) &\leq \alpha
    \end{align}
    which shows that inequality can only be guaranteed when $\alpha_t \leq \frac{\alpha}{d}$.
\end{proof}

\dtocProp*
\begin{proof}
    Starting from the marginal coverage guarantee in the original domain (see Equation \ref{eq:marginal-guarantee}), we aim to show that this extends to the transformed domain. By the construction of the transformed predictive set:  the set $g(\hat{C}_{\alpha}(X_{n+1};Z_{1:n})) = \bigcup_{y \in \hat{C}_{\alpha}(X_{n+1}; Z_{1:n})} g(y)$ contains all possible transformed values $g(y)$ for $y \in \hat{C}_{\alpha}(X_{n+1};Z_{1:n})$. By definition of the transformation \(g\), the event \(g(Y_{n+1}) \cap g\left(\hat{C}_{\alpha}(X_{n+1}; Z_{1:n})\right) \neq \emptyset\) is true if and only if there exists \(y \in \hat{C}_{\alpha}(X_{n+1}; Z_{1:n})\) such that \(g(Y_{n+1})\) overlaps with \(g(y)\). Since \(Y_{n+1} \in \hat{C}_{\alpha}(X_{n+1}; Z_{1:n})\) implies that \(g(Y_{n+1})\) is included in \(\hat{C}'_{\alpha}(X_{n+1}; Z_{1:n})\), the inclusion relationship is preserved:
    \[
    \mathbb{P}_{Z \sim P_{Z}}\left(Y_{n+1} \in \hat{C}_{\alpha}\left(X_{n+1}; Z_{1:n}\right)\right) \leq \mathbb{P}_{Z \sim P_{Z}}\left(g(Y_{n+1}) \cap g\left(\hat{C}_{\alpha}(X_{n+1}; Z_{1:n})\right) \neq \emptyset\right).
    \]

    Finally, since the original coverage guarantee holds for \(Y_{n+1}\) in the untransformed space (Equation \ref{eq:marginal-guarantee}), the transformed set $g\left(\hat{C}_{\alpha}(X_{n+1};Z_{1:n})\right)$ constructed as a union ensures the marginal guarantee in the transformed space:
    \[
    \mathbb{P}_{Z \sim P_{Z}}\left(g(Y_{n+1}) \cap g\left(\hat{C}_{\alpha}(X_{n+1};Z_{1:n})\right) \neq \emptyset\right) \geq 1 - \alpha.
    \]
\end{proof}

\transGuarProp*
\begin{proof}
    The marginal coverage guarantee in Equation \ref{eq:transGuarProp-marginal-guarantee} for monotonic and invertible transformation $g$ are guaranteed because the inclusion relationships are preserved under transformation $g$. For the invertible transformation, this evident invertibility ensures that no information is lost, so $Y_{n+1} \in\hat{C}_{\alpha}(X_{n+1};Z_{1:n}) $ is equivalent to $g(Y_{n+1}) \in g\left(\hat{C}_{\alpha}(X_{n+1};Z_{1:n})\right)$. The same equivalence is present for monotonic transformation since the transformation preserves the relative ordering of elements.  
\end{proof}

\newpage
\section{Extensive literature review}
\subsection{Uncertainty quantification in landmark localization}
\label{ap:uq_landmark}
\subsubsection{Dropout variational inference (Monte Carlo dropout)}
\citet{drevicky_evaluating_2020} utilize Monte Carlo (MC) dropout \cite{gal_dropout_2016} to quantify the uncertainty of a static heatmap regression model, deriving solely approximation (epistemic) uncertainty from the variance in MC sample predictions. 

\citet{lee_automated_2020} adopt a two-stage model combining CNN-based region-of-interest (ROI) detection with a coordinate regression model using a CNN. They apply MC dropout in the second stage, use the variance of the MC samples as an uncertainty estimate, and leverage it for producing 95\% confidence regions, assuming that the uncertainty distribution is a bivariate normal. The reliability of these regions is not validated, which is crucial for use in a clinical context.

Additionally, \citet{jafari_u-land_2022} applies MC dropout to measure pixel-level epistemic uncertainty in static heatmap regression. Epistemic uncertainty is leveraged to detect out-of-distribution (OOD) frames, as the authors believe and demonstrate that OOD frames generally have higher epistemic uncertainty. They demonstrate this on the critical video frame landmark detection tasks in ultrasound imaging videos of the heart. Besides quantifying epistemic uncertainty, in \citet{jafari_u-land_2022}, they also quantify pixel-level aleatoric uncertainty, as they expect that in non-key frames, the regions around an obscured landmark region have high aleatoric uncertainty. They learn the aleatoric uncertainty by performing variational inference on the pixel level.


\subsubsection{Test time augmentations}
\citet{ma_volumetric_2020} introduce a test-time augmentation (TTA) strategy to improve prediction accuracy and robustness while generating 90\% prediction regions. Their method involves applying random translation shifts as the augmentation technique. To form the prediction regions, they calculate the standard deviations of the TTA predictions along the three spatial axes, assuming that the error distribution follows an uncorrelated multivariate Gaussian. The evaluation of uncertainty estimates is limited to examining the correlation between the size of these prediction regions and the localization error.

\subsubsection{Gaussian processes} 
\citet{schobs_bayesian_2023} propose a two-stage method for uncertainty quantification using Gaussian processes. The initial stage employs a static heatmap regression model to generate a coarse landmark prediction. This prediction defines a region of interest, which is input for a Convolutional Gaussian Process (CGP). The CGP then produces a distribution of probable landmark locations, modeled as a bivariate normal distribution, capturing total uncertainty. The authors primarily assess the model's fit to the bivariate normal distribution but do not evaluate the validity of the uncertainty estimates comprehensively. Additionally, this approach's point prediction performance is considerably lower than that of other methods.

\subsubsection{Ensemble learning}
In \citet{drevicky_evaluating_2020, schobs_uncertainty_2023}, they propose to use an ensemble model consisting of independently trained and randomly initialized static heatmap regression models; the uncertainty is quantified by measuring the variance in the landmark predictions of the ensembles. This approach only considers epistemic uncertainty, more specifically, approximation uncertainty. However, it is argued that this approach generally performs better than Monte Carlo dropout methods due to more decorrelated sampled inference models \citep{fort_deep_2020}.

\subsubsection{Maximum likelihood estimation}
\citet{mccouat_contour-hugging_2022} approach landmark localization as a classification problem by utilizing one-hot encoded heatmaps as targets. They implement an encoder-decoder network architecture and apply a 2D softmax function separately to each output channel, with each channel corresponding to a specific landmark's heatmap. To improve model calibration, they use temperature scaling. This approach produces non-parametric, contour-hugging heatmaps closely following anatomical features, such as the skull in cephalometric landmark detection. The authors assess the reliability of pixel-wise probability estimates, which are generally very low (below 0.10), achieving an Expected Calibration Error (ECE) of 0.7. However, since the probabilities are so small, the effectiveness of calibrated confidence regions remains uncertain, as the study does not implement or evaluate these regions. The authors introduce the Expected Radial Error (ERE) as an uncertainty estimate, which shows a strong correlation (0.96) with the true radial error and can be used to identify potentially inaccurate predictions. They use this uncertainty estimate to flag erroneous predictions (ERE $\geq$ 4 mm).

Another approach is to leverage an intermediate heatmap approach and use a maximum likelihood estimation, assuming a specific probability distribution and an estimated covariance matrix. In \citet{kumar_uglli_2019}, they assume a bivariate Gaussian distribution and minimize the negative log-likelihood. They evaluate the uncertainty estimate by plotting the estimated uncertainty (represented as the determinant of the covariance matrix) and the localization error. In \citet{kumar_luvli_2020}, they extend the approach to deal with occluded landmarks.

\subsubsection{Infering uncertainty from heatmaps}
\citet{drevicky_evaluating_2020} introduce normalized maximum heatmap activation (MHA) from a static heatmap regression model as a confidence measure for landmark predictions. A limitation of this method is that it only identifies landmarks at the peak value of the heatmap, lacking subpixel accuracy. Additionally, the confidence metric lacks a true statistical interpretation of uncertainty, instead serving as a correlate of confidence. Determining the source of uncertainty quantified is challenging, though it likely encompasses elements of both aleatoric and epistemic uncertainty. \cite{schobs_uncertainty_2023} enhance this approach by using an ensemble of models to compute an ensemble-MHA, averaging MHA values across ensemble predictions for greater robustness.

In \citet{payer_uncertainty_2020, thaler_modeling_2021}, they propose to learn the Gaussian covariances of target heatmaps such that they represent, as they call, a \textit{dataset-based annotation uncertainty}. This could be seen as homoscedastic aleatoric uncertainty. They learn these covariance parameters of the anisotropic Gaussian function $(\theta_i, \sigma_i^{maj}, \sigma_i^{min})$ simultaneously with weights of the network. For the $i$th landmark, $\theta_i$ represents the rotation of the Gaussian function's major axis, and $\sigma_i^{maj}$ and $\sigma_i^{min}$ present its extent in major and minor axis, respectively, for the $i$th landmark. They add to the mean squared error loss function, which minimizes the predicted $\hat{h}_i$ and target $h_i$ heatmaps, a regularization term which penalizes $\sigma_i^{maj} \sigma_i^{min}$ with $\alpha$ to avoid that $\sigma_i^{maj} \sigma_i^{min} \rightarrow \infty$ and $h_i \approx 0$. The authors also propose quantifying so-called \textit{sample-based annotation uncertainty}, which you could see as a combination of aleatoric and epistemic uncertainty, by fitting a Gaussian function to the predicted landmarks using a robust least squares curve fitting method. Besides the proposed approach uncertainty quantification capabilities, they also show strong performance in terms of accuracy and robustness. They evaluate the reliability of their uncertainty estimates by plotting the relation between the learned Gaussian sizes $\hat{\sigma}_i^{maj} \hat{\sigma}_i^{min}$ and the point prediction errors, but fail to evaluate the calibration of their fitted Gaussian distribution properly.

\citet{payer_uncertainty_2020, thaler_modeling_2021} also try to connect the inter-observer variability of landmark annotation and the learned homoscedastic aleatoric uncertainty, as the latter represents this inter-observer variability. To evaluate this hypothesis, \citet{thaler_modeling_2021} annotated five selected landmarks by a total of 11 annotators for 100 images and compared it with the learned homoscedastic aleatoric uncertainty and the average of the distribution parameters of the sample-based annotation uncertainty. Based on their visual comparisons of the fitted annotation Gaussian distribution, using the 11 annotations and their proposed approach, they conclude that there is a correspondence between them. However, the estimated inter-variability through the learned homoscedastic aleatoric uncertainty and \textit{sample-based annotation uncertainty} is consistently smaller. We believe this is primarily due to an error in their training approach since they trained their proposed method on the mean coordinate of the 11 annotations per image. Following this procedure, you are actually trying to determine the intra-variability of the mean estimate of the 11 annotators, assuming their proposed approach can even model the true aleatoric uncertainty. We believe training the model with augmented datasets would be a better approach to estimating the inter-variability. Each image is presented 11 times with a different annotator label. Additionally, due to the construction of the training procedure and loss function, the parameters of the uncertainty distribution are learned in a heuristic way, even assuming the true aleatoric uncertainty is homoscedastic and follows the bivariate normal distribution. We leave this for future work.

In \citet{kwon_multistage_2021}, they propose to use intermediate heatmaps, which represent probability density function, and a softmax operation \citep{luvizon_2d3d_2018} to get the desired anatomical landmarks. They use these intermediate heatmaps to get confidence regions of the estimated landmark position. They propose to visualize it with standard deviational ellipses of $3\sigma$, thus essentially assuming a bivariate normal distribution.

\subsection{Multi-output conformal prediction}
\label{ap:cp_multi_output}


\subsubsection{Density-based}
\label{ap:denisty-based}
The foundational work by \citet{lei_efficient_2011, lei_distribution-free_2013} introduced the first framework for multi-dimensional target conformal prediction, combining density estimation, e.g., kernel density estimation (KDE), and conformal prediction. Their method uses the estimated density evaluated at $y$ as a nonconformity measure. \citet{lei_conformal_2015} further extend this work to functional data, which was later generalized by \citet{diquigiovanni_importance_2021, diquigiovanni_conformal_2022}. In \citet{english_conformalised_2024}, they proposed using conditional normalizing flows for producing the density estimation to deal with the computational burden, mitigate a grid search, and generate prediction regions; they propose to use Monte Carlo sampling of the multivariate label space.

\subsubsection{Statistical correction}
\label{ap:stat-cor}
Much research has focused on constructing hyperrectangular prediction regions by considering each target dimension as a separate conformal predictor and treating it as the multiple comparison problem prevalent in statistics. \citet{messoudi_conformal_2020} proposed extending single-dimensional nonconformity scores to m-dimensional problems by applying conformal prediction for each dimension separately. To achieve the desired global coverage $(1-\alpha)$, they employed the Šidák correction, setting the individual significance levels to $\alpha_t=1-\sqrt[d]{1-\alpha}$. While computationally straightforward, this approach has limitations. It only maintains exact validity for independent scores (considering the individual conformal prediction approaches have an exact validity guarantee) and provides conservative validity for positively dependent nonconformity scores. A problem arises when the scores are negatively dependent, which can invalidate the prediction regions.

Concurrently, \citet{stankeviciute_conformal_2021} proposed a conformal prediction approach to provide validity guarantees for prediction regions of multi-step time-series forecasting using a Bonferroni correction, setting the individual significance levels to $\alpha_t=\frac{\alpha}{d}$. The validity proof relied on the validity guarantee inductive conformal prediction and Boole's inequality. Note that this often results in highly conservative prediction regions. 

\subsubsection{Copula-based}
\label{ap:copula}
\citet{messoudi_copula-based_2021} introduced a more sophisticated approach than the simple and conservative statistical correction approaches, using copulas to model dependencies between target variables. Their method uses Sklar's theorem to define individual \citep{sklar_fonctions_1959} to define individual significance levels $\alpha_t$ based on global significance level $\alpha$ and $m$-dimensional copula $C$. They proposed three specific copulas: the independent copula (equivalent to Šidák correction), the Gumbel copula (single-parameter estimation), and the empirical copula (non-parametric estimation). A downside of their approach is that validity guarantees are only given when the copula is known. In \citet{mukama_copula-based_2024}, they applied copula-based conformal prediction to the object detection problem.
Copula-based conformal prediction was extended by \citet{zhang_improved_2023}, who allowed different significance values across dimensions by optimizing prediction region volume. \citet{sun_copula_2024} further formalized the approach, providing validity proofs under IID assumptions and using the empirical copula. They achieve validity by using two calibration sets, one for getting the cumulative distribution of the conformity score for each time-step or output variable and the other to estimate the copula. A downside of this approach is that you need quite a large calibration set.

\subsubsection{Ellipsoidal prediction regions}
\label{ap:ellipsoidal}
A parallel research stream focuses on ellipsoidal prediction regions. \citet{kuchibhotla_exchangeability_2021} introduced the Mahalaobis distance, a nonconformity measure, requiring both m-dimensional output predictions and an $m \times m$-dimensional sample inverse-covariance matrix. \citet{johnstone_conformal_2021} formalized ellipsoid conformal prediction regions in the context of robust optimization and full conformal prediction. \citet{messoudi_ellipsoidal_2022} extended the method to generate object-specific ellipsoidal regions by estimating conditional covariance matrices. \citet{henderson_adaptive_2024} proposed the conformal conditional linear expectation method, providing adaptive ellipsoid regions with validity guarantees by leveraging the Mahalanobis distance and ridge estimations of covariance matrix condition on the object $x$. Additionally, they study the asymptotic properties of the ellipsoid.

\subsubsection{Multi-modal prediction regions}
\label{ap:multi-modal}
Several works focus on multi-output regression problems with multi-modal prediction regions. Conventionally, conformal prediction approaches often produce single-connected areas, which can be inefficient when the true conditional distribution is multi-modal. The density-based approach of \citet{lei_efficient_2011} can capture a multi-modal structure. A possible problem of the produced prediction regions of the density-based approach is that they are possibly non-convex and can be a problem for downstream tasks, such as decision-making \citep{johnstone_conformal_2021}. Additionally, the interpretability of these regions for end users can be confusing. Therefore, \citet{tumu_multi-modal_2023} proposed a method for multi-modal prediction regions, ensuring flexible but convex prediction regions,  using a two-stage calibration process using two calibration sets. The first calibration set clusters residuals using KDE and defines shape template functions; the second set combines the shape template into a single nonconformity score. Besides the computational complexity of the method, a downside is that shapes are determined by the calibration set and, thus, are not fully adaptive. 

\citet{wang_probabilistic_2023} introduced Probabilistic Conformal Prediction (PCP), which allows discontinuous prediction regions through a sampling-based approach. Their methodology begins by fitting a conditional generative model on the training data. For calibration, the method generates K-independent prediction samples for each example and calculates the conformity score as the minimum distance between these samples and the true label. When making predictions for test examples, the approach generates K samples from the fitted distribution and constructs prediction regions by centering balls at each sample point, with radii equal to the calibrated conformity score quantile. The final prediction region is formed by taking the union of these individual balls.

\citet{feldman_calibrated_2023} developed a distribution-free approach for multi-modal quantile regression using deep learning techniques. Their method uses a conditional variational autoencoder to learn a representation where the multi-output distribution becomes unimodal. In this transformed space, they construct quantile regions using non-parametric directional quantile regression, which are then mapped back to the original space and conformally calibrated. While the approach effectively handles complex multi-modal structures and provides theoretical guarantees, it requires substantial training data and involves computationally intensive transformation steps that may lose some fine structural details.

\subsubsection{Mixed-variables types}
\label{ap:mixed-variable}
\citet{dheur_distribution-free_2024} addressed bivariate prediction with mixed continuous-discrete variables, proposing the Conformal HDR (C-HDR) method, which generalizes the High Predictive Density approach \citep{izbicki_cd-split_2022}. While they only consider the bivariate case, this generalization can be used for higher dimensions multi-output problems with continuous and discrete outcomes.

\subsubsection{Other approaches}
\label{ap:other-approaches}
Notably, \citet{diquigiovanni_conformal_2022} introduced the maximum nonconformity score $S_i := \max\left(|Y_{i,1}- \hat{Y}_{i,1}|, \ldots, |Y_{i,d}- \hat{Y}_{i,d}|\right)$, which has been leveraged in many subsequent approaches. \citet{dietterich_conformal_2022} built upon this and proposes SBox and SQbox, which adapt the maximum conformity score by replacing the absolute error as measure with a normalized absolute error \citep{papadopoulos_reliable_2011} and CQR \citep{romano_conformalized_2019}. This evolution continued with \citet{cleaveland_conformal_2024}, who proposed a time series approach using linear complementary programming that generates prediction regions over multiple time steps. Their method introduces a weighted maximum nonconformity score $S_i := \max (a_1 S_{i,1},\ldots, a_T S_{i,T})$ across time horizon T, where parameters $a_t$ sum to one and are optimized using two calibration sets—one for parameter optimization to minimize region size, and another for computing nonconformity scores. While this approach yields less conservative regions than the Bonferroni correction, it requires re-optimization for different significance levels.

\clearpage
\newpage
\section{Landmark localization performance}
\label{ap:det-performance}
\input{det_results_isbi2015}
\input{det_results_chd}
\input{det_results_mml_benchmark}
\input{det_results_mml_benchmark_adj}

\clearpage
\section{Additional experiments results}
\label{ap:additional-results}
\begin{figure}[htp]
    \centering
    \includegraphics[width=\textwidth]{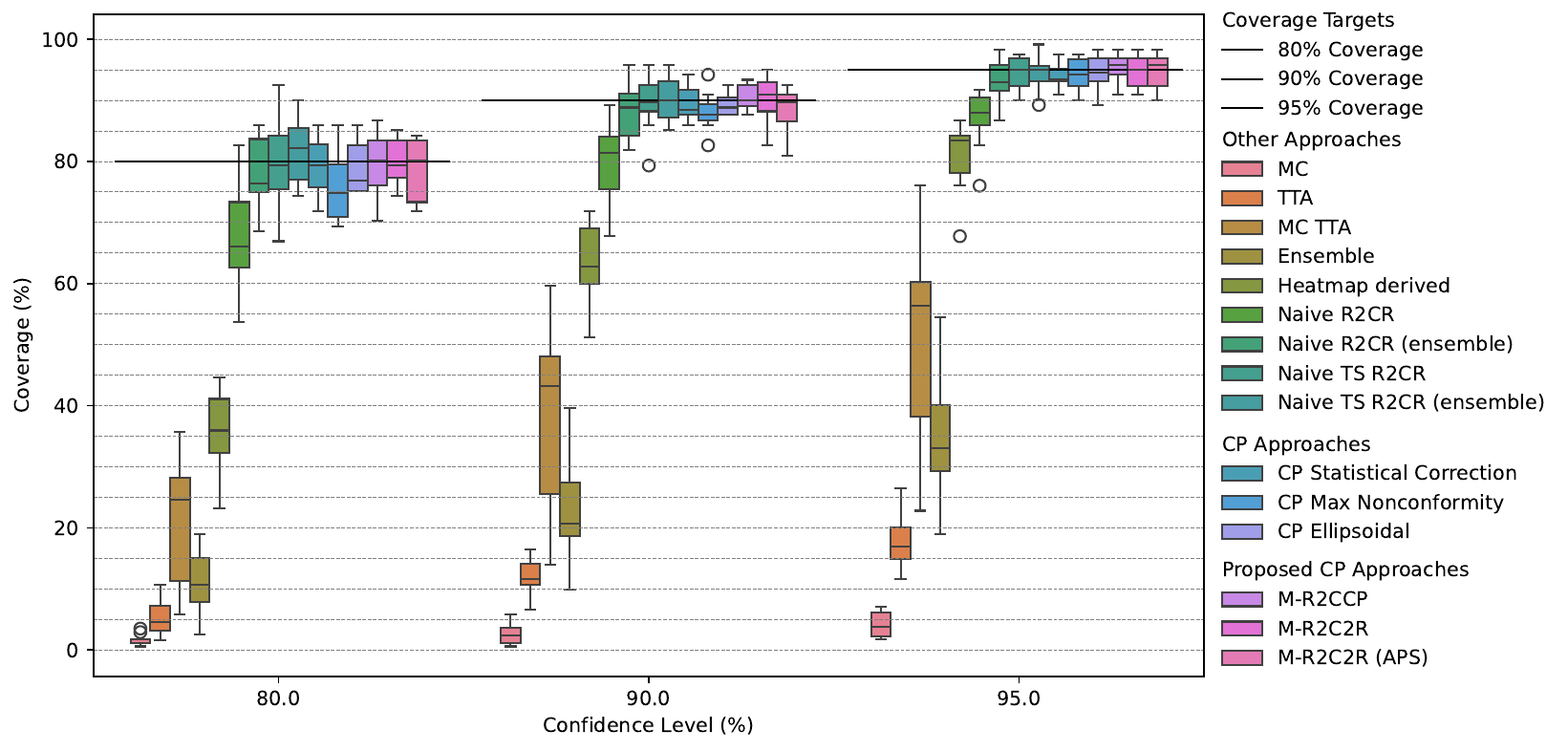}
    \caption{Box plot showing the empirical coverage of prediction regions for different landmarks (12 landmarks), evaluating various uncertainty quantification approaches at different confidence target levels on the CHD (2D) dataset test set.}
    \label{fig:coverage-chd}
\end{figure}

\input{uq_results_chd}

\begin{figure}[htp]
    \centering
    \includegraphics[width=\textwidth]{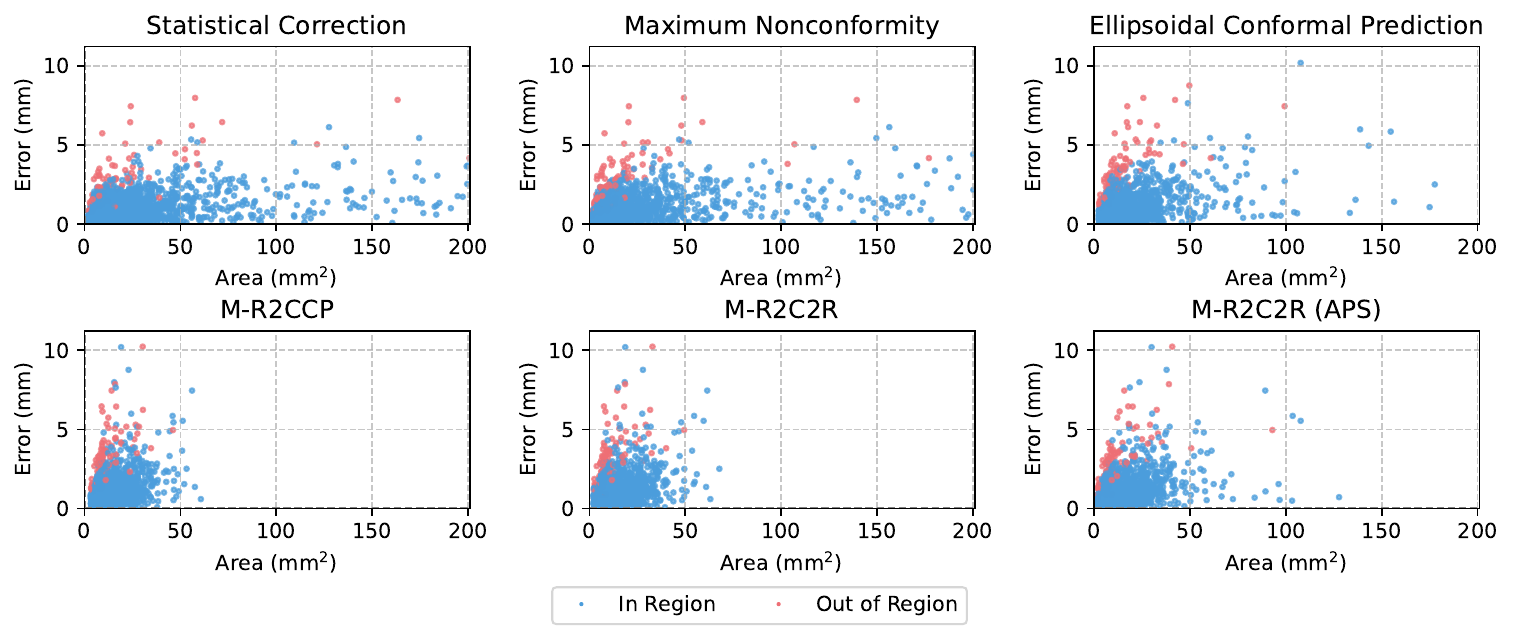}
    \caption{Scatter plot that illustrates the relationship between prediction error and efficiency (area) on the ISBI 2015 (2D) dataset.}
    \label{fig:isbi2015-adaptivity}
\end{figure}

\begin{figure}[htp]
    \centering
    \includegraphics[width=\textwidth]{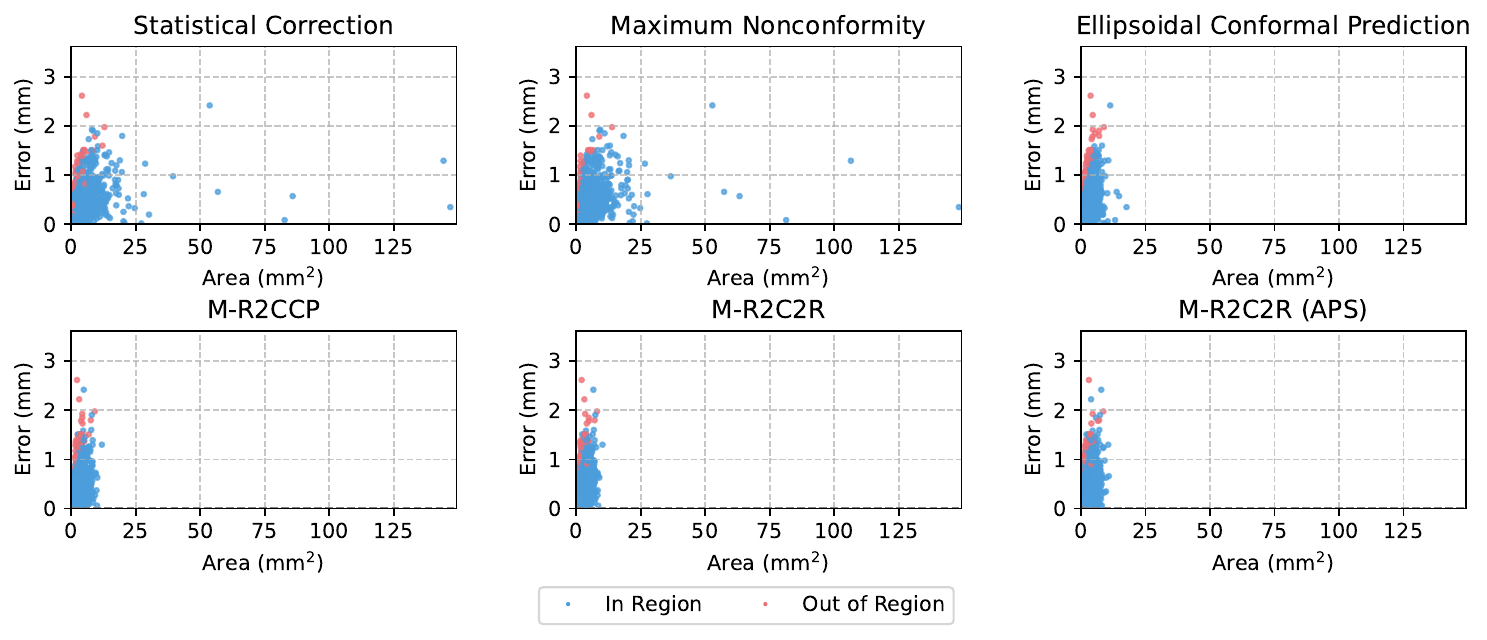}
    \caption{Scatter plot that illustrates the relationship between prediction error and efficiency (area) on the CHD (2D) dataset.}
    \label{fig:chd-adaptivity}
\end{figure}

\begin{figure}[htp]
    \centering
    \includegraphics[width=\textwidth]{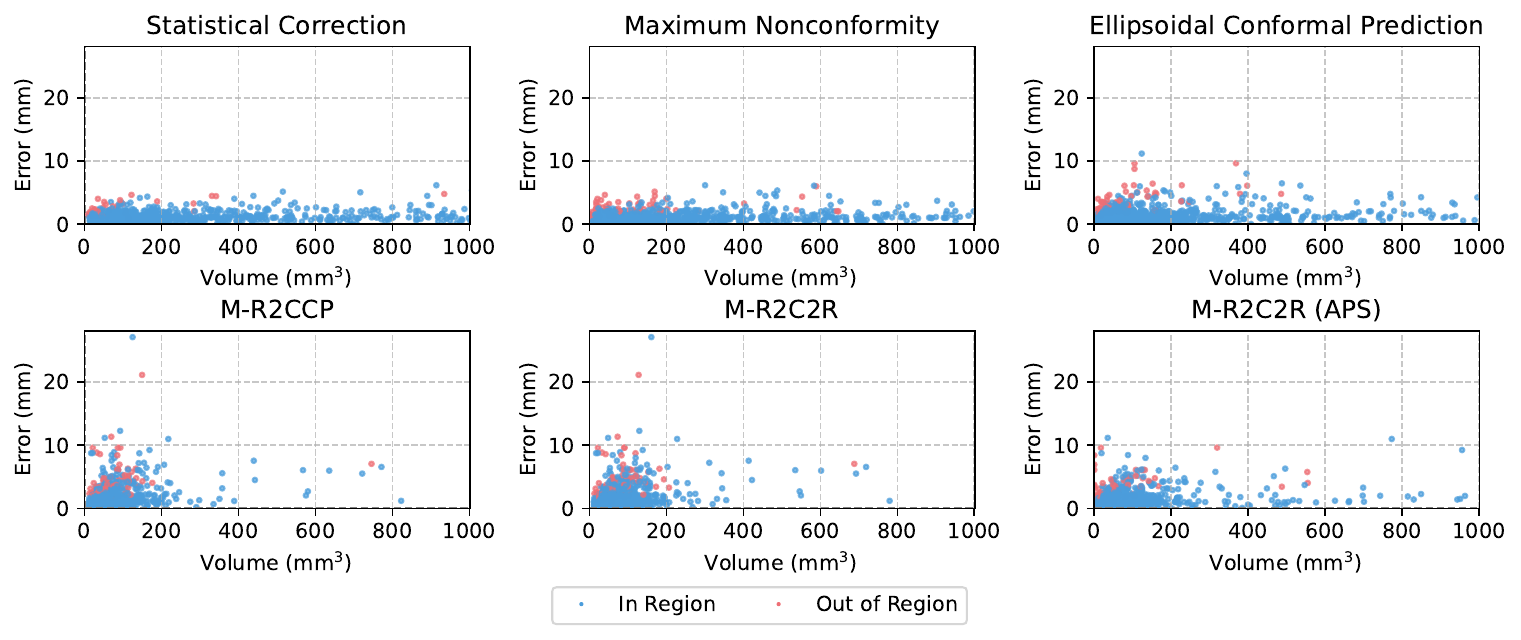}
    \caption{Scatter plot that illustrates the relationship between prediction error and efficiency (volume) on the MML (3D) dataset.}
    \label{fig:mml-adaptivity}
\end{figure}

\begin{figure}[htp]
    \centering
    \includegraphics[width=0.8\textwidth]{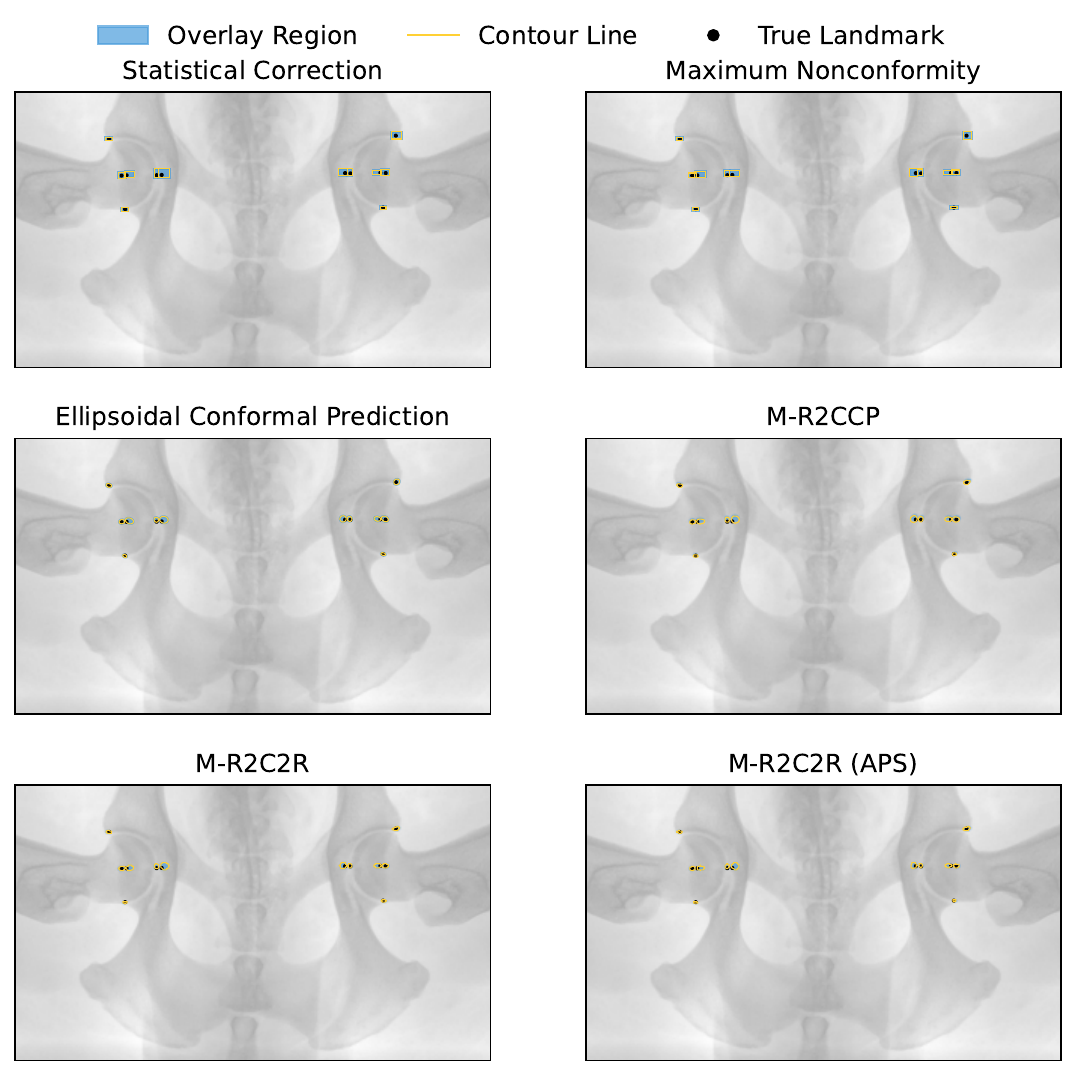}
    \caption{Canine pelvis example of conformal prediction regions.}
    \label{fig:canine-cp-example}
\end{figure}

\end{document}

%% file: overview_works.tex

\definecolor{headerBg}{RGB}{245,245,245}        
\definecolor{alternateBg}{RGB}{252,252,252}     
\definecolor{tableBorder}{RGB}{220,220,220}     

\newcommand{\checkicon}{\raisebox{-0.3\height}{\includegraphics[height=14pt]{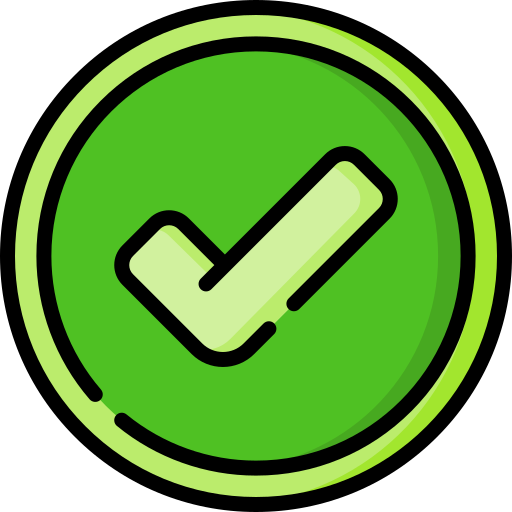}}}
\newcommand{\crossicon}{\raisebox{-0.3\height}{\includegraphics[height=14pt]{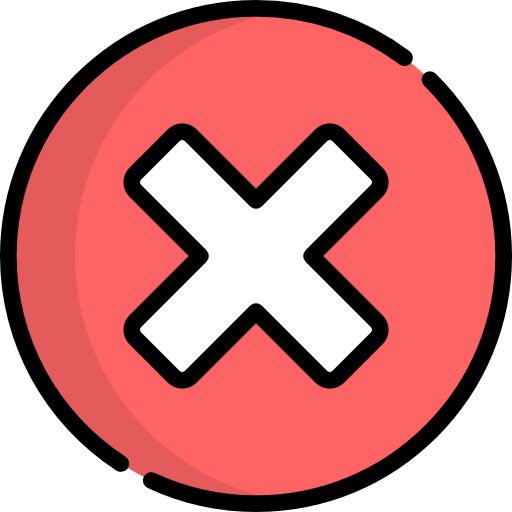}}}
\newcommand{\dashicon}{\raisebox{-0.3\height}{\includegraphics[height=14pt]{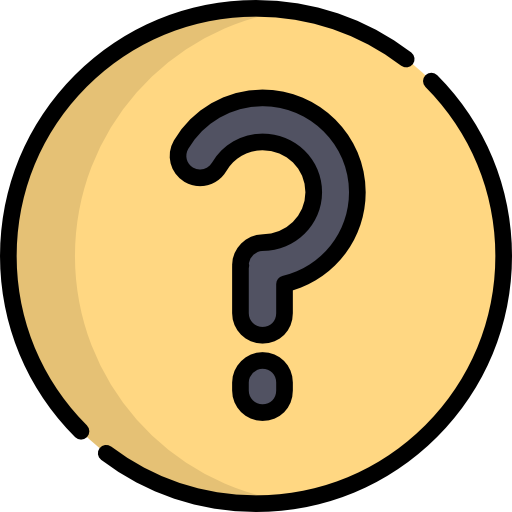}}}

\newcommand{\ccell}[1]{\multicolumn{1}{c}{\makecell[c]{#1}}}

\begin{table*}[!htbp]
\setlength{\aboverulesep}{0pt}
\setlength{\belowrulesep}{0pt}
\setlength{\tabcolsep}{6pt}  
\renewcommand{\arraystretch}{1.3}  
\centering

\caption{Overview of related works on uncertainty quantification for anatomical landmark localization.}
\label{tab:overview-works}
\begin{adjustbox}{width=1\textwidth}
\begin{tabular}{
    >{\raggedright\arraybackslash}m{2.5cm}
    >{\raggedright\arraybackslash}m{2.7cm}
    >{\raggedright\arraybackslash}m{1.8cm}
    ccccc
    c
}
\toprule[\heavyrulewidth]
\rowcolor{headerBg}
\multicolumn{1}{c}{} & 
\multicolumn{1}{c}{} & 
\multicolumn{1}{c}{} & 
\multicolumn{1}{c}{} & 
\multicolumn{2}{c}{\textbf{Epistemic}} & 
\multicolumn{3}{c}{} \\

\rowcolor{headerBg}
\multicolumn{1}{c}{\textbf{Work}} & 
\multicolumn{1}{c}{\textbf{Approach}} & 
\multicolumn{1}{c}{\makecell[c]{\textbf{Uncertainty}\\\textbf{repres.}}} & 
\multicolumn{1}{c}{\textbf{Aleatoric}} & 
\multicolumn{1}{c}{\textbf{Model}} & 
\multicolumn{1}{c}{\textbf{Approx.}} & 
\multicolumn{1}{c}{\makecell[c]{\textbf{Validity}\\\textbf{eval.}}} & 
\multicolumn{1}{c}{\makecell[c]{\textbf{Adaptivity}\\\textbf{eval.}}} & 
\multicolumn{1}{c}{\textbf{Dim.}} \\

\midrule[\heavyrulewidth]

\rowcolor{white}
\cite{lee_automated_2020} &
MC-dropout &
Variance, region &
\crossicon &
\crossicon &
\checkicon &
\crossicon &
\crossicon &
2D \\

\rowcolor{alternateBg}
\cite{drevicky_evaluating_2020} &
MC-dropout, ensemble, heatmap derived &
Variance &
\crossicon &
\crossicon &
\checkicon &
\crossicon &
\crossicon &
2D \\

\rowcolor{white}
\cite{jafari_u-land_2022} &
MC-dropout, variational inference &
Variance &
\checkicon &
\crossicon &
\checkicon &
\crossicon &
\crossicon &
2D \\

\rowcolor{alternateBg}
\cite{ma_volumetric_2020} &
TTA &
Region &
\checkicon &
\crossicon &
\crossicon &
\crossicon &
\checkicon &
3D \\

\rowcolor{white}
\cite{schobs_uncertainty_2023} &
Ensemble, heatmap derived &
Variance &
\dashicon &
\dashicon &
\dashicon &
\checkicon &
\checkicon &
2D \\

\rowcolor{alternateBg}
\cite{schobs_bayesian_2023} &
Gaussian process &
Covariance &
\checkicon &
\crossicon &
\checkicon &
\crossicon &
\crossicon &
2D \\

\rowcolor{white}
\cite{kwon_multistage_2021} &
Heatmap derived &
Covariance &
\checkicon &
\crossicon &
\crossicon &
\crossicon &
\crossicon &
2D \\

\rowcolor{alternateBg}
\cite{payer_uncertainty_2020, thaler_modeling_2021} &
Heatmap derived &
Covariance, entropy &
\dashicon &
\dashicon &
\dashicon &
\crossicon &
\crossicon &
2D/3D \\

\rowcolor{white}
\cite{kumar_uglli_2019, kumar_luvli_2020} &
MLE &
Covariance &
\checkicon &
\crossicon &
\crossicon &
\crossicon &
\checkicon &
2D \\

\rowcolor{alternateBg}
\cite{mccouat_contour-hugging_2022} &
MLE, heatmap derived &
Expected radial error &
\checkicon &
\crossicon &
\crossicon &
\checkicon &
\checkicon &
2D \\

\bottomrule[\heavyrulewidth]
\end{tabular}
\end{adjustbox}
\vspace{1em}
\begin{minipage}{\textwidth}
\small
\hspace{1.5em} \checkicon~Yes \hspace{1.5em} \crossicon~No \hspace{1.5em} \dashicon~Not Clear
\end{minipage}

\end{table*}

%% file: sota_mml.tex
\begin{table}[htp]
\centering
\caption{Landmark localization results on the MML (3D) dataset compared against existing approaches. This benchmark dataset is the same as the data
subset, only with complete landmarks, used as a benchmark in \citet{he_anchor_2024}. The best results are highlighted in bold.}
\label{tab:mml-sota}
\resizebox{\textwidth}{!}{
\begin{tabular}{llllll|lllll}
\hline
Model name        & \multicolumn{5}{c|}{Validation}                          & \multicolumn{5}{c}{Test}                        \\ \cline{2-11} 
                  & PE (mm) & \multicolumn{4}{c|}{SDR (\%)}                  & PE (mm) & \multicolumn{4}{c}{SDR (\%)}          \\ \cline{3-6} \cline{8-11} 
                  &         & 2 mm    & 2.5 mm  & 3 mm    & 4 mm             & (mm)    & 2 mm    & 2.5 mm  & 3 mm    & 4 mm    \\ \hline
\begin{tabular}[c]{@{}l@{}}Pruning-ResUNet3D\\ \citep{he_anchor_2024}\end{tabular} & 1.82    & 73.21\% & 82.14\% & 88.93\% & \textbf{94.76\%} & 1.96    & 70.03\% & 79.97\% & 86.10\% & 92.73\% \\ \hline
\textbf{One hot ensemble} &
  \textbf{1.60} &
  \textbf{77.81\%} &
  \textbf{87.76\%} &
  \textbf{89.92\%} &
  93.37\% &
  \textbf{1.39} &
  \textbf{81.67\%} &
  \textbf{91.31\%} &
  \textbf{93.33\%} &
  \textbf{96.31\%} \\ \hline
\end{tabular}
}
\end{table}

%% file: uq_results_isbi2015.tex
\begin{table}[htp]
\centering
\caption{Evaluation of 95\% prediction regions of different uncertainty quantification (UQ) approaches on the ISBI 2015 (2D) test set. The best results are highlighted in bold.}
\label{tab:isbi2015-uq}
\begin{tabular}{@{}llllllll@{}}
\toprule
UQ Approach              & Coverage (\%) & \multicolumn{5}{c}{Efficiency ($\text{mm}^2$)} & Adaptivity ($r_s$) \\ \cmidrule(lr){3-7}
                      &         & Mean            & Std.   & Median          & Q1     & Q3     &                 \\ \midrule
MC                    & 5.47\%  & 0.3015          & 1.2333 & 0.0862          & 0.0472 & 0.1836 & 0.2642          \\
TTA                   & 20.68\% & 2.7351          & 50.134 & 0.5793          & 0.3694 & 0.9650 & 0.3699          \\
MC TTA                & 25.47\% & 1.4625          & 3.9245 & 0.5889          & 0.3778 & 1.0562 & 0.3732          \\
Ensemble              & 25.21\% & 5.0906          & 65.185 & 0.7687          & 0.4734 & 1.6475 & \textbf{0.4056} \\
Heatmap derived       & 78.63\% & 5.8981          & 5.5784 & 4.2713          & 3.2454 & 6.6309 & 0.3958          \\
Naive R2CR            & 86.26\% & 19.950          & 506.75 & 5.7200          & 3.9200 & 9.0150 & 0.4310          \\
Naive R2CR (ensemble) & 92.47\% & 9.1205          & 6.9787 & 6.9300          & 4.8200 & 11.035 & 0.3867          \\
Naive TS R2CR         & 95.32\% & 676.41          & 3959.4 & 10.305          & 6.4650 & 17.940 & 0.3860          \\
Naive TS R2CR (ensemble) & 94.26\%       & 10.921    & 9.3519    & 8.0900   & 5.6300   & 13.263   & 0.3773             \\ \hdashline
CP Statistical Correction         & 95.53\% & 55.291          & 245.29 & 14.620          & 7.9068 & 32.157 & 0.3703          \\
CP Max Nonconformity & 93.42\% & 44.982          & 226.92 & 11.465          & 6.4124 & 23.266 & 0.3833          \\
CP Ellipsoidal        & 94.63\% & 15.812          & 18.067 & 9.7825          & 7.2443 & 18.094 & 0.3385          \\ \hdashline
M-R2CCP               & 94.63\% & \textbf{12.081} & 8.0110 & 9.0000          & 7.2200 & 14.348 & 0.3360          \\
M-R2C2R               & 94.84\% & 12.202          & 9.1147 & \textbf{8.3350} & 6.6300 & 15.163 & 0.3127          \\
M-R2C2R (APS)         & 94.63\% & 13.213          & 13.278 & 8.5900          & 6.3175 & 16.463 & 0.3528          \\ \bottomrule
\end{tabular}
\end{table}

%% file: uq_results_mml.tex
\begin{table}[htp]
\centering
\caption{Evaluation of 95\% prediction regions of different uncertainty quantification (UQ) approaches on the MML (3D) dataset test set. The best results are highlighted in bold.}
\label{tab:mml-uq}
\begin{tabular}{@{}llllllll@{}}
\toprule
UQ Approach & Coverage (\%) & \multicolumn{5}{c}{Efficiency ($\text{mm}^3)$} & Adaptivity ($r_s$) \\ \cmidrule(lr){3-7}
                         &         & Mean            & Std.    & Median          & Q1     & Q3     &                 \\ \midrule
MC                       & 29.50\% & 232.50          & 1776.9  & 2.1958          & 0.8617 & 8.9424 & \textbf{0.5794} \\
Ensemble                 & 13.36\% & 63.108          & 574.84  & 0.2323          & 0.0732 & 1.3626 & 0.4437 \\
Heatmap derived          & 81.86\% & 626.91          & 3453.1  & 21.555          & 9.2268 & 69.694 & 0.3015          \\
Naive R2CR               & 89.86\% & 1739.4          & 9990.2  & 39.653          & 22.812 & 75.474 & 0.3547          \\
Naive R2CR (ensemble)    & 90.29\% & 1661.5          & 10607   & 29.803          & 20.506 & 53.173 & 0.3395          \\
Naive TS R2CR            & 92.57\% & 3133.1          & 13212   & 52.697          & 28.582 & 116.43 & 0.3499          \\
Naive TS R2CR (ensemble) & 87.50\% & 572.89          & 4783.6  & 23.842          & 16.767 & 41.376 & 0.3415          \\ \hdashline
CP Statistical Correction            & 97.43\% & 3.6e+07         & 5.6e+08 & 542.33          & 109.93 & 7733.2 & 0.2652          \\
CP Max Nonconformity    & 95.36\% & 2.8e+07         & 4.1e+08 & 483.26          & 101.36 & 6759.5 & 0.2520          \\
CP Ellipsoidal           & 95.00\% & 2971.7          & 19684   & 73.406          & 33.235 & 251.00 & 0.2869          \\ \hdashline
M-R2CCP                  & 93.21\% & \textbf{66.044} & 65.039  & 51.268          & 33.728 & 76.817 & 0.2756          \\
M-R2C2R                  & 93.71\% & 68.323          & 62.958  & 53.617 & 34.928 & 82.819 & 0.2922          \\
M-R2C2R (APS)            & 93.07\% & 4440.1          & 24056   & \textbf{51.067} & 27.707 & 94.819 & 0.2765          \\ \bottomrule
\end{tabular}
\end{table}

%% file: det_results_isbi2015.tex
\begin{table}[htp]
\centering
\caption{Anatomical landmark localization performance results on ISBI 2015 (2D) dataset. The best results are highlighted in bold.}
\label{tab:isbi-det}
\begin{adjustbox}{width=1\textwidth}
\begin{tabular}{@{}llllll|lllll@{}}
\toprule
Model name          & \multicolumn{5}{c|}{Validation}                                        & \multicolumn{5}{c}{Test}                                   \\ \cmidrule(l){2-11} 
                    & PE (mm) & \multicolumn{4}{c|}{SDR (\%)}                                & PE (mm) & \multicolumn{4}{c}{SDR (\%)}                     \\ \cmidrule(lr){3-6} \cmidrule(l){8-11} 
                    &         & 2 mm         & 2.5 mm       & 3 mm         & 4 mm            &         & 2 mm         & 2.5 mm & 3 mm   & 4 mm            \\ \midrule
SCN                 & 1.1446  & 86.26\%       & 90.53\%       & 93.84\%       & 97.21\%          & 1.2051  & 83.58\%       & 89.68\% & 93.32\% & 96.84\%          \\
SCN ensemble        & 1.101   & 87.21\%       & 91.16\%       & 94.42\%       & 97.58\%          & 1.1672  & 84.95\%       & 90.68\% & 93.89\% & 97.16\%          \\
SCN ensemble LS &
  {\ul \textbf{1.084}} &
  87.21\% &
  91.53\% &
  95.11\% &
  97.74\% &
  {\ul 1.1463} &
  85.37\% &
  {\ul 91.37\%} &
  {\ul 94.47\%} &
  {\ul 97.42\%} \\
SCN MC              & 1.1364  & 86.11\%       & 90.63\%       & 94.05\%       & 97.32\%          & 1.2065  & 83.47\%       & 89.84\% & 93.11\% & 96.95\%          \\
SCN MC LS           & 1.1255  & 86.58\%       & 91.00\%         & 94.47\%       & 97.32\%          & 1.1888  & 84.53\%       & 90.32\% & 94.00\%   & 97.11\%          \\
SCN MC TTA          & 1.0977  & {\ul 87.37\%} & 91.16\%       & 94.58\%       & 97.47\%          & 1.1617  & 85.05\%       & 90.42\% & 93.68\% & 96.89\%          \\
SCN MC TTA LS       & 1.1029  & 86.95\%       & 92.11\%       & 94.68\%       & 97.63\%          & 1.163   & {\ul 85.58\%} & 90.79\% & 93.58\% & 97.42\%          \\
SCN TTA             & 1.0939  & 87.21\%       & {\ul 91.68\%} & 94.42\%       & 97.53\%          & 1.1612  & 84.95\%       & 90.42\% & 93.42\% & 97.11\%          \\
SCN TTA LS          & 1.1128  & 86.26\%       & 91.74\%       & {\ul 94.68\%} & {\ul 97.68\%}    & 1.1654  & 85.00\%         & 90.63\% & 93.74\% & 97.16\%          \\ \midrule
One hot             & 1.3149  & 84.26\%       & 90.47\%       & 93.16\%       & 96.58\%          & 1.2502  & 84.53\%       & 90.11\% & 93.58\% & 96.68\%          \\
\textbf{One hot ensemble} &
  {\ul 1.0861} &
  {\ul \textbf{88.11\%}} &
  {\ul \textbf{93.16\%}} &
  {\ul \textbf{95.95\%}} &
  {\ul \textbf{98.32\%}} &
  {\ul \textbf{1.116}} &
  {\ul \textbf{87.42\%}} &
  {\ul \textbf{93.00\%}} &
  {\ul \textbf{95.53\%}} &
  {\ul \textbf{98.05\%}} \\
One hot ensemble LS & 1.123   & 87.37\%       & 92.47\%       & 95.11\%       & 97.68\% & 1.1366  & 86.74\%       & 92.37\% & 94.89\% & 97.84\% \\
One hot TTA         & 1.1958  & 86.32\%       & 90.95\%       & 94.05\%       & 97.16\%          & 1.1564  & 86.21\%       & 91.37\% & 93.89\% & 97.21\%          \\
One hot TTA LS      & 1.1958  & 86.42\%       & 91.74\%       & 94.42\%       & 97.47\%          & 1.1688  & 86.11\%       & 91.00\%   & 94.05\% & 97.42\%          \\ \bottomrule
\end{tabular}
\end{adjustbox}
\end{table}

%% file: det_results_chd.tex
\begin{table}[htp]
\caption{Anatomical landmark localization performance results on CHD (2D) dataset. The best results are highlighted in bold.}
\label{tab:chd-det}
\begin{adjustbox}{width=1\textwidth}
\begin{tabular}{llllll|lllll}
\hline
Model name &
  \multicolumn{5}{c|}{Validation} &
  \multicolumn{5}{c}{Test} \\ \cline{2-11} 
 &
  PE (mm) &
  \multicolumn{4}{c|}{SDR (\%)} &
  PE (mm) &
  \multicolumn{4}{c}{SDR (\%)} \\ \cline{3-6} \cline{8-11} 
 &
   &
  2 mm &
  2.5 mm &
  3 mm &
  4 mm &
   &
  2 mm &
  2.5 mm &
  3 mm &
  4 mm \\ \hline
SCN &
  0.4591 &
  99.20\% &
  99.68\% &
  99.84\% &
  100.0\% &
  0.4867 &
  99.32\% &
  99.76\% &
  99.95\% &
  100.0\% \\
SCN ENSEMBLE &
  0.4358 &
  99.52\% &
  99.84\% &
  99.84\% &
  100.0\% &
  0.4670 &
  99.51\% &
  100.0\% &
  100.0\% &
  100.0\% \\
SCN ENSEMBLE LS &
  0.4231 &
  {\ul 99.76\%} &
  {\ul 100.0\%} &
  {\ul 100.0\%} &
  {\ul 100.0\%} &
  0.4569 &
  99.76\% &
  100.0\% &
  100.0\% &
  100.0\% \\
SCN MC &
  0.4599 &
  99.20\% &
  99.68\% &
  99.84\% &
  100.0\% &
  0.4863 &
  99.27\% &
  99.76\% &
  99.95\% &
  100.0\% \\
SCN MC LS &
  0.4519 &
  99.36\% &
  99.84\% &
  99.92\% &
  100.0\% &
  0.4816 &
  99.37\% &
  99.85\% &
  100.0\% &
  100.0\% \\
SCN MC TTA &
  0.3934 &
  99.36\% &
  99.84\% &
  99.92\% &
  100.0\% &
  0.4161 &
  99.61\% &
  99.95\% &
  100.0\% &
  100.0\% \\
SCN MC TTA LS &
  0.3775 &
  99.68\% &
  99.92\% &
  100.0\% &
  100.0\% &
  0.4096 &
  {\ul 99.95\%} &
  {\ul 100.0\%} &
  {\ul 100.0\%} &
  {\ul 100.0\%} \\
SCN TTA &
  {\ul 0.3912} &
  99.28\% &
  99.76\% &
  100.0\% &
  100.0\% &
  0.4131 &
  99.76\% &
  99.90\% &
  99.95\% &
  100.0\% \\
SCN TTA LS &
  0.3772 &
  99.52\% &
  99.92\% &
  100.0\% &
  100.0\% &
  {\ul \textbf{0.4059}} &
  {\ul \textbf{99.95\%}} &
  {\ul \textbf{100.0\%}} &
  {\ul \textbf{100.0\%}} &
  {\ul \textbf{100.0\%}} \\ \hline
One hot &
  0.6115 &
  99.24\% &
  99.58\% &
  99.65\% &
  99.79\% &
  0.5601 &
  99.17\% &
  99.72\% &
  99.86\% &
  99.86\% \\
\textbf{One hot ensemble} &
  {\ul \textbf{0.4835}} &
  {\ul \textbf{99.72\%}} &
  {\ul \textbf{99.79\%}} &
  {\ul \textbf{99.93\%}} &
  {\ul \textbf{99.93\%}} &
  {\ul 0.4750} &
  {\ul \textbf{99.79\%}} &
  {\ul \textbf{99.93\%}} &
  {\ul \textbf{100.0\%}} &
  {\ul \textbf{100.0\%}} \\
One hot ensemble LS &
  0.5646 &
  98.89\% &
  98.96\% &
  99.17\% &
  99.24\% &
  0.4987 &
  99.38\% &
  99.52\% &
  99.59\% &
  99.66\% \\
One hot TTA &
  0.6179 &
  99.31\% &
  99.72\% &
  99.72\% &
  99.79\% &
  0.5672 &
  99.10\% &
  99.72\% &
  99.86\% &
  99.86\% \\
One hot TTA LS &
  0.5998 &
  98.96\% &
  99.31\% &
  99.51\% &
  99.58\% &
  0.5630 &
  99.31\% &
  99.79\% &
  99.93\% &
  99.93\% \\ \hline
\end{tabular}
\end{adjustbox}
\end{table}

%% file: det_results_mml_benchmark.tex
\begin{table}[htp]
\centering
\caption{Anatomical landmark localization performance results on MML (3D) dataset. This benchmark dataset is the same as the data subset, only with complete landmarks, used as a benchmark in \citet{he_anchor_2024}. The best results are highlighted in bold.}
\label{tab:mml-det-benchmark}
\begin{adjustbox}{width=1\textwidth}
\begin{tabular}{@{}llllll|lllll@{}}
\toprule
Model Name    & \multicolumn{5}{c|}{Validation}            & \multicolumn{5}{c}{Test}                   \\ \cmidrule(l){2-11} 
              & PE (mm)   & \multicolumn{4}{c|}{SDR (\%)}     & PE (mm)     & \multicolumn{4}{c}{SDR (\%)}      \\ \cmidrule(lr){3-6} \cmidrule(l){8-11} 
              &   & 2 mm   & 2.5 mm & 3 mm   & 4 mm   &    & 2 mm   & 2.5 mm & 3 mm   & 4 mm   \\ \midrule
One hot       & 1.983  & 71.81\% & 80.23\% & 84.57\% & 90.31\% & 1.6938 & 77.50\%  & 86.43\% & 90.24\% & 94.29\% \\
\textbf{One hot ensemble} &
  \textbf{1.601} &
  \textbf{77.81\%} &
  \textbf{87.76\%} &
  \textbf{89.92\%} &
  93.37\% &
  \textbf{1.3918} &
  \textbf{81.67\%} &
  \textbf{91.31\%} &
  \textbf{93.33\%} &
  96.31\% \\
One hot ensemble LS &
  1.632 &
  76.02\% &
  85.46\% &
  89.41\% &
  \textbf{93.62\%} &
  1.4694 &
  81.67\% &
  89.40\% &
  93.10\% &
  \textbf{96.55\%} \\
One hot MC    & 2.0249 & 73.09\% & 84.06\% & 86.86\% & 91.20\%  & 1.6392 & 77.98\% & 87.98\% & 90.48\% & 94.05\% \\
One hot MC LS & 2.0557 & 71.30\%  & 79.21\% & 84.57\% & 90.18\% & 1.6433 & 77.98\% & 86.07\% & 90.12\% & 95.24\% \\ \bottomrule
\end{tabular}
\end{adjustbox}
\end{table}

%% file: det_results_mml_benchmark_adj.tex
\begin{table}[htp]
\centering
\caption{Anatomical landmark localization performance results on MML (3D) dataset with adjusted pixel spacing. This benchmark dataset is the same as the data subset, only with complete landmarks, used as a benchmark in \citet{he_anchor_2024}. The best results are highlighted in bold.}
\label{tab:mml-det-adj-benchmark}
\begin{adjustbox}{width=1\textwidth}
\begin{tabular}{@{}llllll|lllll@{}}
\toprule
Model name    & \multicolumn{5}{c|}{Validation}             & \multicolumn{5}{c}{Test}                    \\ \cmidrule(l){2-11} 
              & PE (mm) & \multicolumn{4}{c|}{SDR (\%)}     & PE (mm) & \multicolumn{4}{c}{SDR (\%)}      \\ \cmidrule(lr){3-6} \cmidrule(l){8-11} 
              &         & 2 mm   & 2.5 mm & 3 mm   & 4 mm   &         & 2 mm   & 2.5 mm & 3 mm   & 4 mm   \\ \midrule
One hot       & 2.1221  & 69.90\%  & 78.32\% & 83.93\% & 89.41\% & 1.6937  & 77.50\%  & 86.43\% & 90.24\% & 94.29\% \\
\textbf{One hot ensemble} &
  \textbf{1.6947} &
  \textbf{75.64\%} &
  \textbf{86.61\%} &
  \textbf{89.29\%} &
  \textbf{92.86\%} &
  \textbf{1.3922} &
  \textbf{81.67\%} &
  \textbf{91.31\%} &
  \textbf{93.21\%} &
  96.31\% \\
One hot ensemble LS &
  1.7366 &
  74.11\% &
  83.42\% &
  88.14\% &
  \textbf{92.86\%} &
  1.4693 &
  81.67\% &
  89.40\% &
  93.10\% &
  \textbf{96.55\%} \\
One hot MC    & 2.2365  & 71.05\% & 82.14\% & 85.46\% & 89.41\% & 1.6769  & 75.48\% & 86.79\% & 89.29\% & 92.74\% \\
One hot MC LS & 2.1655  & 69.13\% & 78.70\%  & 83.93\% & 89.29\% & 1.6556  & 78.10\%  & 85.36\% & 90.24\% & 95.24\% \\ \bottomrule
\end{tabular}
\end{adjustbox}
\end{table}

%% file: uq_results_chd.tex
\begin{table}[htp]
\centering
\caption{Evaluation of 95\% prediction regions of different UQ approaches on the CHD (2D) dataset test set. The best results are highlighted in bold.}
\label{tab:chd-uq}
\begin{tabular}{llllllll}
\hline
UQ Approach              & Coverage (\%) & \multicolumn{5}{c}{Efficiency ($\text{mm}^2$)} & Adaptivity ($r_s$) \\ \cline{3-7}
                      &         & Mean   & Std.   & Median          & Q1     & Q3     &                 \\ \hline
MC                    & 4.19\%  & 0.0642 & 0.5059 & 0.0209          & 0.0104 & 0.0485 & 0.3288          \\
TTA                   & 17.98\% & 0.3225 & 2.0733 & 0.2108          & 0.1429 & 0.2888 & 0.3704          \\
MC TTA                & 51.66\% & 0.5910 & 0.8375 & 0.4254          & 0.1794 & 0.7188 & 0.1400          \\
Ensemble              & 34.85\% & 1.5587 & 29.318 & 0.3726          & 0.1538 & 0.6308 & 0.3781 \\
Heatmap derived       & 80.79\% & 1.4402 & 0.6872 & 1.3137          & 0.9721 & 1.8107 & 0.3883          \\
Naive R2CR            & 87.26\% & 2.2778 & 1.3138 & 2.0287          & 1.3138 & 2.9754 & \textbf{0.4025}          \\
Naive R2CR (ensemble) & 93.25\% & 2.4321 & 1.3035 & 2.2026          & 1.4684 & 3.2073 & 0.3888          \\
Naive TS R2CR         & 94.63\% & 31.480 & 405.40 & 3.3232          & 2.1446 & 4.7336 & 0.3922          \\
Naive TS R2CR (ensemble) & 94.56\%       & 2.7101    & 1.4821    & 2.4634   & 1.5843   & 3.5551   & 0.3805             \\ \hdashline
CP  Statistical Correction         & 94.01\% & 4.6423 & 7.2967 & 3.2219          & 2.0695 & 5.3252 & 0.3392          \\
CP Max Nonconformity & 94.28\% & 4.6088 & 6.7433 & 3.1491          & 1.9605 & 5.4012 & 0.3606          \\
CP Ellipsoidal        & 94.84\% & 3.0813 & 1.7368 & 2.6978          & 1.9480 & 3.8588 & 0.3369          \\ \hdashline
M-R2CCP               & 95.59\% & 2.9391 & 1.6140 & 2.4538          & 1.8935 & 3.6710 & 0.3604          \\
M-R2C2R               & 94.63\% & \textbf{2.7559} & 1.5110 & 2.3475 & 1.7002 & 3.5164 & 0.3373          \\
M-R2C2R (APS)         & 94.97\% & 2.7687 & 1.6471 & \textbf{2.2992 }         & 1.6230 & 3.5792 & 0.3498          \\ \hline
\end{tabular}
\end{table}